\documentclass{article}

\usepackage{microtype}
\usepackage{graphicx}
\usepackage{subfig}
\usepackage{booktabs}

\usepackage{multirow}
\usepackage{amsmath}
\usepackage{amsthm}
\usepackage{amssymb}
\usepackage{empheq}
\usepackage[switch]{lineno}
\usepackage[hyphens]{url}
\usepackage{caption}

\newcommand{\eat}[1]{}

\newcommand{\D}{\mathcal{D}}
\newcommand{\dr}{$d_{\mathrm{rmax}}$}

\newtheorem{theorem}{Theorem}[section]
\newtheorem*{theorem*}{Theorem}

\newtheorem{lemma}[theorem]{Lemma}
\newtheorem*{lemma*}{Lemma}

\usepackage{hyperref}

\usepackage[noend]{algorithmic}

\newlength\myindent
\usepackage[accepted]{icml2021}

\icmltitlerunning{Machine Unlearning for Random Forests}

\newcommand{\bigO}{\mathcal{O}}

\begin{document}

\twocolumn[
\icmltitle{Machine Unlearning for Random Forests}




\begin{icmlauthorlist}
\icmlauthor{Jonathan Brophy}{uo}
\icmlauthor{Daniel Lowd}{uo}
\end{icmlauthorlist}

\icmlaffiliation{uo}{Department of Computer and Information Science, University of Oregon, Eugene, Oregon}

\icmlcorrespondingauthor{Jonathan Brophy}{jbrophy@cs.uoregon.edu}

\icmlkeywords{Machine Learning, ICML, Machine Unlearning, Random Forests, Data Deletion, Privacy}

\vskip 0.3in
]



\printAffiliationsAndNotice{}  

\begin{abstract}
Responding to user data deletion requests, removing noisy examples, or deleting corrupted training data are just a few reasons for wanting to delete instances from a machine learning~(ML) model. However, efficiently removing this data from an ML model is generally difficult. In this paper, we introduce data removal-enabled (DaRE) forests, a variant of random forests that enables the removal of training data with minimal retraining. Model updates for each DaRE tree in the forest are exact, meaning that removing instances from a DaRE model yields exactly the same model as retraining from scratch on updated data.

DaRE trees use randomness and caching to make data deletion efficient. The upper levels of DaRE trees use random nodes, which choose split attributes and thresholds uniformly at random. These nodes rarely require updates because they only minimally depend on the data. At the lower levels, splits are chosen to greedily optimize a split criterion such as Gini index or mutual information. DaRE trees cache statistics at each node and training data at each leaf, so that only the necessary subtrees are updated as data is removed. For numerical attributes, greedy nodes optimize over a random subset of thresholds, so that they can maintain statistics while approximating the optimal threshold. By adjusting the number of thresholds considered for greedy nodes, and the number of random nodes, DaRE trees can trade off between more accurate predictions and more efficient updates.

In experiments on 13 real-world datasets and one synthetic dataset, we find DaRE forests delete data orders of magnitude faster than retraining from scratch while sacrificing little to no predictive power.
\end{abstract}

\section{Introduction}

Recent legislation \cite{gdpr,california_cpa,pipeda} requiring companies to remove private user data upon request has prompted new discussions on data privacy and ownership~\cite{shintre2019making}, and fulfilling this ``right to be forgotten''~\cite{kwak2017let,garg2020formalizing} may require updating any models trained on this data~\cite{villaronga2018humans}. However, retraining a model from scratch on a revised dataset becomes prohibitively expensive as dataset sizes and model complexities increase~\cite{shoeybi2019megatron}; the result is wasted time and computational resources, exacerbated as the frequency of data removal requests increases.

Decision trees and random forests~\cite{breiman1984classification,friedman2001greedy} are popular and widely used machine learning models~\cite{lundberg2018consistent}, mainly due to their predictive prowess on many classification and regression tasks~\cite{kocev2013tree,genuer2017random,wager2018estimation,linero2018bayesian,biau2019neural}. Current work on deleting data from machine learning models has focused mainly on recommender systems~\cite{cao2015towards,schelter20}, K-means~\cite{ginart2019making}, SVMs~\cite{cauwenberghs2001incremental}, logistic regression~\cite{guo2020certified,schelter20}, and deep neural networks~\cite{baumhauer2020machine,golatkar2020forgetting,wu2020deltagrad};
however, there is very limited work addressing the problem of efficient data deletion for tree-based models~\cite{schelter2021hedgecut}.
Thus, we outline our contributions as follows:
\begin{enumerate}

\item We introduce DaRE (\textbf{Da}ta \textbf{R}emoval-\textbf{E}nabled) Forests (a.k.a DaRE RF), a variant of random forests that supports the efficient removal of training instances. DaRE RF works with discrete tree structures, in contrast to many related works on efficient data deletion that assume continuous parameters. The key components of DaRE RF are to retrain subtrees only as needed, consider only a subset of valid thresholds per attribute at each decision node, and to strategically place completely random nodes near the top of each tree to avoid costly retraining.
    
\item We provide algorithms for training and subsequently removing data from a DaRE forest.
    
\item We evaluate DaRE RF's ability to efficiently perform sequences of deletions on 13 real-world binary classification datasets and one synthetic dataset, and find that DaRE RF can typically delete data 2-4 orders of magnitude faster than retraining from scratch while sacrificing less than 1\% in terms of predictive performance.


\end{enumerate}

\section{Problem Formulation}


We assume an instance space $\mathcal{X} \subseteq \mathbb{R}^p$ and possible labels $\mathcal{Y} = \{+1, -1\}$\footnote{Our methods can easily be generalized to the multi-class setting, $|\mathcal{Y}| = C$, by storing statistics for $C-1$ classes instead of just one.}. Let $\mathcal{D} = \{(x_i, y_i)\}_{i=1}^{n}$ be a training dataset in which each instance $x_i \in \mathcal{X}$ is a $p$-dimensional vector ($x_{i,j}$)$_{j=1}^{p}$ and $y_i \in \mathcal{Y}$. We refer to $P = \{j\}_{j=1}^{p}$ as the set of possible attributes. 


\subsection{Unlearning}

Our goal is to ``unlearn'' specific training examples by updating a trained model to completely remove their influence. We base our definition on prior work by~\citet[Def.~3.1]{ginart2019making}. We define a (possibly randomized) \emph{learning algorithm}, $\mathcal{A}: \D \rightarrow \mathcal{H}$, as a function from a dataset $\D$ to a model in hypothesis space $\mathcal{H}$.
A \emph{removal method}, $\mathcal{R}: \mathcal{A}(\D) \times \D \times (\mathcal{X} \times \mathcal {Y}) \rightarrow \mathcal{H}$, is a function from a model $\mathcal{A}(\D)$, dataset $\D$, and an instance to remove from the training data $(x, y)$ to a model in $\mathcal{H}$. For \emph{exact unlearning}~(a.k.a.\ \emph{perfect unlearning}), the removal method must be equivalent to applying the training algorithm to the dataset with instance $(x,y)$ removed. In the case of randomized training algorithms, we define equivalence as having identical probabilities for each model in $\mathcal{H}$:
\begin{equation}
\label{eq:exact_unlearning}
P(\mathcal{A}(\D \setminus (x,y))) = P(\mathcal{R}(\mathcal{A}(\D), \D, (x, y)))
\end{equation}
See \S\ref{sec:related_work} for more related work on unlearning.

The simplest approach to exact unlearning is to ignore the existing model and simply rerun $\mathcal{A}$ on the updated dataset, $\D \setminus (x,y)$. We refer to this as the \emph{naive retraining} approach. Naive retraining is agnostic to virtually all machine learning models, easy to understand, and easy to implement. However, this approach becomes prohibitively expensive as the dataset size, model complexity, and number of deletion requests increase.

\subsection{Random Forests}
A \emph{decision tree} is a tree-structured model in which each leaf is associated with a binary-valued prediction and each internal node is a decision node associated with an attribute $a \in P$ and threshold value $v \in \mathbb{R}$. The outgoing branches of the decision node partition the data based on the chosen attribute and threshold. Given $x \in \mathcal{X}$, the prediction of a decision tree can be found by traversing the tree, starting at the root and following the branches consistent with the attribute values in $x$. Traversal ends at one of the leaf nodes, where the prediction is equal to the value of the leaf node.

Decision trees are typically learned in a recursive manner, beginning by picking an attribute $a$ and threshold $v$ at the root that optimizes an empirical split criterion such as the Gini index~\cite{breiman1984classification}:
\begin{align}\label{eq:gini_index}
G_{D,\mathcal{Y}}(a, v) = \sum_{b \in \{\ell,r\}} \frac{|D_{b}|}{|D|}
             \Bigg(1 - \sum_{y \in \mathcal{Y}}
             \bigg(\frac{|D_{b,y}|}{|D_{b}|}\bigg)^2\Bigg)
\end{align}
or entropy~\cite{quinlan2014c4}:
\begin{equation}\label{eq:entropy}
H_{D,\mathcal{Y}}(a, v) = \sum_{b \in \{\ell,r\}} \frac{|D_{b}|}{|D|}
             \Bigg(\sum_{y \in \mathcal{Y}}
             -\frac{|D_{b,y}|}{|D_{b}|}
             \log_2 \frac{|D_{b,y}|}{|D_{b}|}
             \Bigg)
\end{equation}
in which $D \subseteq \D$ is the input dataset to a decision node, $D_{\ell} = \{(x_i, y_i) \in D \mid x_{i, a} \leq v\}$, $D_{r} = D \setminus D_{\ell}$, $D_{\ell, y} = \{(x_i, y_i) \in D_{\ell} \mid y_i=y\}$, and $D_{r, y} = \{(x_i, y_i) \in D_{r} \mid y_i=y\}$. Once $a$ and $v$ have been chosen for the root node, the data is partitioned into mutually exclusive subsets based on the value of $v$, and a child node is learned for each data subset. The process terminates when the entire subset has the same label or the tree reaches a specified maximum depth $d_{\max}$.

A {\em random forest}~(RF) is an ensemble of decision trees which predicts the mean value of its component trees. Two sources of randomness are used to increase diversity among the trees. First, each tree in the ensemble is trained from a bootstrap sample of the original training data, with some instances excluded and some included multiple times. Second, each decision node is restricted to a random subset of attributes, and the split criterion is optimized over this subset rather than over all attributes.

We base our methods on a minor variation of a standard RF, one that does not use bootstrapping. Bootstrapping complicates the removal of training instances, since one instance may appear multiple times in the training data for one tree. There is also empirical evidence that bootstrapping does not improve predictive performance~\cite{zaman2009effect,denil2014narrowing,mentch2016quantifying}, which was consistent with our own experiments~(Appendix: \S\ref{appendix_subsec:predictive_performance}, Table~\ref{tab:predictive_performance}). Since predictive performance was already similar, we saw no need to add the extra bookkeeping to handle this complexity.


\section{DaRE Forests}

We now describe DaRE (\textbf{Da}ta \textbf{R}emoval-\textbf{E}nabled) forests~(a.k.a. DaRE RF), an RF variant that enables the efficient removal of training instances.

\begin{theorem}\label{theorem:deletion_correctness}
Data deletion for DaRE forests is \emph{exact}~(see Eq.~\ref{eq:exact_unlearning}), meaning that removing instances from a DaRE model yields exactly the same model as retraining from scratch on updated data.
\end{theorem}

This is also equivalent to certified removal~\cite{guo2020certified} with $\epsilon = 0$. Proofs to all theorems are in \S\ref{appendix_sec:algorithmic_details} of the Appendix.

A DaRE forest is a tree ensemble in which each tree is trained independently on a copy of the training data, considering a random subset of $\Tilde{p}$ attributes at each split to encourage diversity among the trees. In our experiments we use $\Tilde{p} = \lfloor \sqrt{p} \rfloor$. Since each tree is trained independently, we describe our methods in terms of training and updating a single tree; the extension to the ensemble is trivial.

DaRE forests leverage several techniques to make deletions efficient: (1) only retrain portions of the model where the structure must change to match the updated database; (2) consider at most $k$ randomly-selected thresholds per attribute; (3) introduce random nodes at the \emph{top} of each tree that minimally depend on the data and thus rarely need to be retrained. We present abridged versions for training and updating a DaRE tree in Algorithms~1 and 2, respectively, with full explanations below. Detailed pseudocode for both operations is in the Appendix,~\S\ref{appendix_subsec:pseduocode}.

\subsection{Retraining Minimal Subtrees}

We avoid unnecessary retraining by storing statistics at each node in the tree. For decision nodes, we store and update counts for the number of instances $|D|$ and positives instances $|D_{\cdot, 1}|$,
as well as $|D_{\ell}|$ and $|D_{\ell,1}|$ for a set of $k$ thresholds per attribute. This information is sufficient to recompute the split criterion of each threshold without iterating through the data.
For leaf nodes, we store and update $|D|$ and $|D_{\cdot, 1}|$, along with a list of training instances that end at that leaf. These statistics are initialized when training the tree for the first time~(Alg.~\ref{alg:training}). We find this additional overhead has a negligible effect on training time.

When deleting a training instance $(x, y) \in D$, these statistics are updated and used to check if a particular subtree needs retraining. Specifically, decision nodes affected by the deletion of $(x, y)$ update the statistics and recompute the split criterion for each attribute-threshold pair. If a different threshold obtains an improved split criterion over the currently chosen threshold, then we retrain the subtree rooted at this node. The training data for this subtree can be found by concatenating the instance lists from all leaf-node descendants. If no retraining occurs at any decision node and a leaf node is reached instead, its label counts and instance list are updated and the deletion operation is complete. See Alg.~\ref{alg:update} for pseudocode.

\begin{algorithm}[t]
\small
\caption{Building a DaRE tree / subtree.}
\label{alg:training}
\begin{algorithmic}[1]
    \STATE {\bfseries Input:} data $D$, depth $d$
    \IF{\text{stopping criteria reached}}
        \STATE{$node$ $\gets$ \textsc{LeafNode}()}
        \STATE{\text{save instance counts}($node$,  $D$)}
        \hfill{$\vartriangleright$~\text{$|D|$, $|D_{\cdot, 1}|$}}
        \STATE{\text{save leaf-instance pointers}($node$,  $D$)}
        \STATE{\text{compute leaf value}($node$)}
    \ELSE
        \IF{$d <$ \dr}
            \STATE{$node \gets$ \textsc{RandomNode}()}
            \STATE{\text{save instance counts}($node$,  $D$)}
            \hfill{$\vartriangleright$~\text{$|D|$, $|D_{\cdot, 1}|$}}
            \STATE{$a \gets$ \text{randomly sample attribute}($D$)}
            \STATE{$v \gets$ \text{randomly sample threshold $\in [a_{min}, a_{max})$}}
            \STATE{\text{save threshold statistics}($node$, $D$, $a$, $v$)}
        \ELSE
            \STATE{$node$ $\gets$ \textsc{GreedyNode}()}
            \STATE{save instance counts($node$,  $D$)}
            \hfill{$\vartriangleright$~\text{$|D|$, $|D_{\cdot, 1}|$}}
            \STATE{$A \gets$ \text{randomly sample $\tilde{p}$ attributes}($D$)}
            \FOR{$a \in$ $A$}
                \STATE{$C$ $\gets$ get valid thresholds($D$, $a$)}
                \STATE{$V$ $\gets$ randomly sample $k$ valid thresholds($C$)}
                \FOR{$v \in V$}
                    \STATE{save threshold statistics($node$, $D$, $a$, $v$)}
                \ENDFOR
            \ENDFOR
            \STATE{$scores \gets$ compute split scores($node$)}
            \STATE{\text{select optimal split}($node$, $scores$)}
        \ENDIF
        \STATE{$D.\ell, D.r \gets$ split on selected threshold($node$, $D$)}
        \STATE{$node.\ell$ = \textsc{Train}($D_{\ell}$, $d+1$)}
        \hfill{$\vartriangleright$~\text{Alg.~\ref{alg:training}}}
        \STATE{$node.r$ $\gets$ \textsc{Train}($D_{r}$, $d+1$)}
        \hfill{$\vartriangleright$~\text{Alg.~\ref{alg:training}}}
    \ENDIF
    \STATE {\bfseries Return} node
\end{algorithmic}
\end{algorithm}

\subsection{Sampling Valid Thresholds}

The optimal threshold for a continuous attribute will always lie between two training instances with adjacent feature values containing opposite labels; if the two training instances have the same label, the split criterion improves by increasing or decreasing $v$. We refer to these as \emph{valid} thresholds, and any other threshold as \emph{invalid}.
More precisely, a threshold $v$ between two adjacent values $v_1$ and $v_2$ for a given attribute $a$ is valid if and only if there exist instances $(x_1,y_1)$ and $(x_2,y_2)$ such that $x_{1,a} = v_1$, $x_{2,a} = v_2$, and $y_1 \neq y_2$.

Only considering valid thresholds substantially reduces the statistics we need to store and compute at each node. We gain further efficiency by randomly sampling $k$ valid thresholds and only considering these thresholds when deciding which attribute-threshold pair to split on. We treat $k$ as a hyperparameter and tune its value when building a DaRE model. One might suspect that only considering a subset of thresholds for each attribute may lead to decreased predictive performance; however, our experiments show that relatively modest values of $k$~(e.g. $5 \leq k \leq 25$) are sufficient to providing accurate predictions, and in some cases lead to improved performance~(Appendix: \S\ref{appendix_subsec:predictive_performance}, Table~\ref{tab:predictive_performance}).

When deleting an instance at a given node, we must determine if any threshold has become invalid. To accomplish this efficiently, at each node we also save and update the number of instances in which attribute $a$ equals $v_1$, the number in which $a$ equals $v_2$, and the number of positive instances matching each of those criteria. When an attribute threshold becomes invalid, we sort and iterate through the node data $D$, resampling the invalid threshold to obtain a new valid threshold.

\begin{algorithm}[tb!]
\small
\caption{Deleting a training instance from a DaRE tree.}\label{alg:update}
\begin{algorithmic}[1]
    \REQUIRE{Start at the root node.}
    \STATE {\bfseries Input:} $node$, depth $d$, instance to remove $(x, y)$.
    \STATE{\text{update instance counts}($node$, $(x, y)$)}
    \hfill{$\vartriangleright$~\text{$|D|$ and $|D_{\cdot, 1}|$}}
    \IF{$node$ is a \textsc{LeafNode}}
        \STATE{\text{remove $(x, y)$ from leaf-instance pointers($node$, $(x,y)$})}
        \STATE{\text{recompute leaf value}($node$)}
        \STATE{\text{remove $(x, y)$ from database and return}}
    \ELSE
        \STATE{\text{update decision node statistics}($node$, $(x, y)$)}
        \IF{$node$ is a \textsc{RandomNode}}
            \IF{$node$.selected threshold is invalid}
                \STATE{$D \gets$ get data from leaf instances($node$) $\setminus$ $(x, y)$}
                \IF{$node$.selected attribute ($a$) is not constant}
                    \STATE{$v \gets$ resample threshold $\in [a_{min}, a_{max})$}
                    \STATE{$D.\ell, D.r \gets$ split on new threshold($node$, $D$, $a$, $v$)}
                    \STATE{$node.\ell, r$ $\gets$ \textsc{Train}($D.\ell, d+1$), \textsc{Train}($D.r, d+1$)}
                \ELSE
                    \STATE{$node \gets$ \textsc{Train}($D$, $d$)}
                    \hfill{$\vartriangleright$~\text{Alg.~\ref{alg:training}}}
                \ENDIF
                \STATE{remove $(x, y)$ from database and return}
            \ENDIF
        \ELSE
            \IF{$\exists$ invalid attributes or thresholds}
                \STATE{$D \gets$ get data from leaf instances($node$) $\setminus$ $(x, y)$}
                \STATE{resample invalid attributes and thesholds($node$, $D$)}
            \ENDIF
            \STATE{$scores$ $\gets$ recompute split scores($node$)}
            \STATE{$a$, $v$ $\gets$ select optimal split($node$, $scores$)}
            \IF{optimal split has changed}
                \STATE{$D.\ell, D.r \gets$ split on new threshold($node$, $D$, $a$, $v$)}
                \STATE{$node.\ell, r$ $\gets$ \textsc{Train}($D.\ell, d+1$), \textsc{Train}($D.r, d+1$)}
                \STATE{remove $(x, y)$ from database and return}
            \ENDIF
        \ENDIF
        \IF{$x_{\cdot, a} \leq v$}
            \STATE{\textsc{Delete}($node.\ell$, $d+1$, $(x, y)$)}
            \hfill{$\vartriangleright$~\text{Alg.~\ref{alg:update}}}
        \ELSE
            \STATE{\textsc{Delete}($node.r$, $d+1$, $(x, y)$)}
            \hfill{$\vartriangleright$~\text{Alg.~\ref{alg:update}}}
        \ENDIF
    \ENDIF
\end{algorithmic}
\end{algorithm}

\subsection{Random Splits}


The third technique for efficient model updating is to choose the attribute and threshold for some of the decision nodes at random, independent of the split criterion. Specifically, given the data at a particular decision node $D \subseteq \D$, we sample an attribute $a \in P$ uniformly at random, and then sample a threshold $v$ in the range $[a_{\min}, a_{\max})$, the min.\ and max.\ values for $a$ in $D$. We henceforth refer to these decision nodes as ``random'' nodes, in contrast to the ``greedy'' decision nodes that optimize the split criterion. Random nodes store and update $|D_\ell|$ and $|D_r|$, statistics based on the sampled threshold, and retrain only if $|D_\ell|=0$ or $|D_r|=0$~(i.e. $v$ is no longer in the range $[a_{\min}, a_{\max})$); however, since random nodes minimally depend on the statistics of the data, they rarely need to be retrained. Random nodes are placed in the upper layers of the tree and greedy nodes are used for all other layers~(excluding leaf nodes). We introduce \dr{} as another hyperparameter indicating how many layers from the top the tree should use for random nodes~(e.g. the top two layers of the tree are all random nodes if \dr$=2$).

Intuitively, nodes near the top of the tree contain more instances than nodes near the bottom, making them more expensive to retrain if necessary. Thus, we can significantly increase deletion efficiency by replacing those nodes with random ones. We can also maintain comparable predictive performance to a model with no random nodes by using greedy nodes in all subsequent layers, resulting in a greedy model built on top of a random projection of the input space~\cite{haupt2006signal}.

In our experiments, we compare DaRE RF with random splits to those without, to evaluate the benefits of adding these random nodes. We refer to DaRE models with random nodes as random DaRE (R-DaRE) and those without as greedy DaRE (G-DaRE). G-DaRE RF can also be viewed as a special case of R-DaRE RF in which \dr$=0$.

\subsection{Complexity Analysis}

The time for training a DaRE forest is \emph{identical} to that of a standard RF:
\begin{theorem}\label{theorem:training_runtime}
Given $n = |\D|$, $T$, $d_{\max}$, and $\Tilde{p}$, the time complexity to train a DaRE forest is $\mathcal{O}(T\, \Tilde{p}\,n\,d_{\max})$.
\end{theorem}

The overhead of storing statistics and instance pointers is negligible compared to the cost of iterating through the entire dataset to score all attributes at each node. The key difference is in the deletion time, which can be much better depending on how much of each tree needs to be retrained:
\begin{theorem}\label{theorem:deletion_runtime}
Given $d_{\max}$, $\Tilde{p}$, and $k$, the time complexity to delete a single instance $(x,y) \in \D$ from a DaRE tree is $\mathcal{O}(\Tilde{p}\,k\,d_{\max})$, if the tree structure is unchanged and the attribute thresholds remain valid. If a node with $|D|$ instances has invalid attribute thresholds, then the additional time to choose new thresholds is $\mathcal{O}(|D| \log |D|)$. If a node with $|D|$ instances at level $d$ needs to be retrained, then the additional retraining time is $\mathcal{O}(\Tilde{p}\,|D|\,(d_{\max}-d)))$.
\end{theorem}

When the structure is unchanged, this is much more efficient than naive retraining, especially if the number of thresholds considered ($k$) is much smaller than $n$. In the worst case, if the split changes at the root of every tree, then deletion in a DaRE forest is no better than naive retraining. In practice, this is very unlikely, since different trees in the forest consider different sets of $\Tilde{p}$ attributes at the root, and the difference between the best and second-best attribute-threshold pairs is usually bigger than a single instance. 

Choosing new thresholds also requires iterating through the training instances at a node. Thresholds only become invalid when an instance adjacent to the threshold is removed, so this is an infrequent event when $k$ is much smaller than $n$. To analyze this empirically, we evaluate our methods with both random and adversarially chosen deletions, approximating the average- and worst-case, respectively.

The main storage costs for a DaRE forest come from storing sets of attribute-threshold statistics at each greedy node, and the instance lists for the leaf nodes.

\begin{theorem}\label{theorem:space_complexity}
Given $n=|\D|$, $d_{\max}$, $k$, $T$, and $\Tilde{p}$, the space complexity of a DaRE forest is $\mathcal{O}(k\,\Tilde{p}\,2^{d_{\max}}\,T + n\,T$).
\end{theorem}

In our experiments, we analyze the space overhead of a DARE forest by measuring its memory consumption as compared to a standard RF, quantifying the time/space trade-off introduced by DARE RF to enable efficient data deletion.

\section{Experimental Evaluation}

\paragraph{Research Questions}
Can we use G-DaRE RF to efficiently delete a significant number of instances as compared to naive retraining~(\textbf{RQ1})? Can we use R-DaRE RF to further increase deletion efficiency 
while maintaining comparable predictive performance~(\textbf{RQ2})? 

\paragraph{Datasets}

We conduct our experiments on 13 publicly-available datasets that represent problems well-suited for tree-based models, and one synthetic dataset we call Synthetic. For each dataset, we generate one-hot encodings for any categorical variable and leave all numeric and binary variables as is. For any dataset without a designated train and test split, we randomly sample 80\% of the data for training and use the rest for testing. A summary of the datasets is in Table~\ref{tab:dataset_summary}, and additional dataset details are in the Appendix: \S\ref{sec:datasets_appendix}.

\paragraph{Hyperparameter Tuning}

Due to the range of label imbalances in our datasets~(Table~\ref{tab:dataset_summary} and Appendix: \S\ref{sec:datasets_appendix}, Table~\ref{tab:datasets}) we measure the predictive peformance of our models using average precision (AP)~\cite{zhu2004department} for datasets with a positive label percentage $<$ 1\%, AUC~\cite{hanley1982meaning} for datasets between [1\%, 20\%], and accuracy~(acc.) for the remaining datasets. Using these metrics and Gini index as the split criterion, we tune the following hyperparameters: the maximum depth of each tree~$d_{\max}$, the number of trees in the forest~$T$, and the number of thresholds considered per attribute for greedy nodes~$k$. Our protocol for tuning \dr\ is as follows: first, we tune a greedy model~(i.e. by keeping $d_{\mathrm{rmax}}=0$ fixed) using 5-fold cross-validation. Once the optimal values for $d_{\max}$, $T$, and $k$ are found, we tune \dr\ by incrementing its value from zero to $d_{\max}$, stopping when the model's cross-validation score exceeds a specified error tolerance as compared to the greedy model; for these experiments, we tune \dr\ using absolute error tolerances of 0.1\%, 0.25\%, 0.5\%, and 1.0\%. Selected hyperparameter values are in the Appendix: \S\ref{appendix_subsec:predictive_performance}, Table~\ref{tab:dart_hyperparameters}.

\subsection{Methodology}

\begin{table}[t]
\centering
\caption{Dataset Summary. $n$ = no.\ instances, $p$ = no.\ attributes, Pos. \% = positive label percentage, Met. = predictive performance metric.}
\vskip 0.15in
\label{tab:dataset_summary}
\begin{tabular}{lrrrr}
\toprule
\textbf{Dataset} & \textbf{$n$} & \textbf{$p$} & \textbf{Pos. \%} & \textbf{Met.} \\
\midrule
Surgical       & 14,635     & 90    & 25.2\% & Acc. \\
Vaccine        & 26,707     & 185   & 46.4\% & Acc. \\
Adult          & 48,842     & 107   & 23.9\% & Acc. \\
Bank Mktg.     & 41,188     & 63    & 11.3\% & AUC  \\
Flight Delays  & 100,000    & 648   & 19.0\% & AUC  \\
Diabetes       & 101,766    & 253   & 46.1\% & Acc. \\
No Show        & 110,527    & 99    & 20.2\% & AUC  \\
Olympics       & 206,165    & 1,004 & 14.6\% & AUC  \\
Census         & 299,285    & 408   &  6.2\% & AUC  \\
Credit Card    & 284,807    & 29    &  0.2\% & AP   \\
CTR            & 1,000,000  & 13    &  2.9\% & AUC  \\
Twitter        & 1,000,000  & 15    & 17.0\% & AUC  \\
Synthetic      & 1,000,000  & 40    & 50.0\% & Acc. \\
Higgs          & 11,000,000 & 28    & 53.0\% & Acc. \\
\bottomrule
\end{tabular}
\end{table}

\begin{figure*}[t]
\centering
\includegraphics[width=1.0\textwidth]{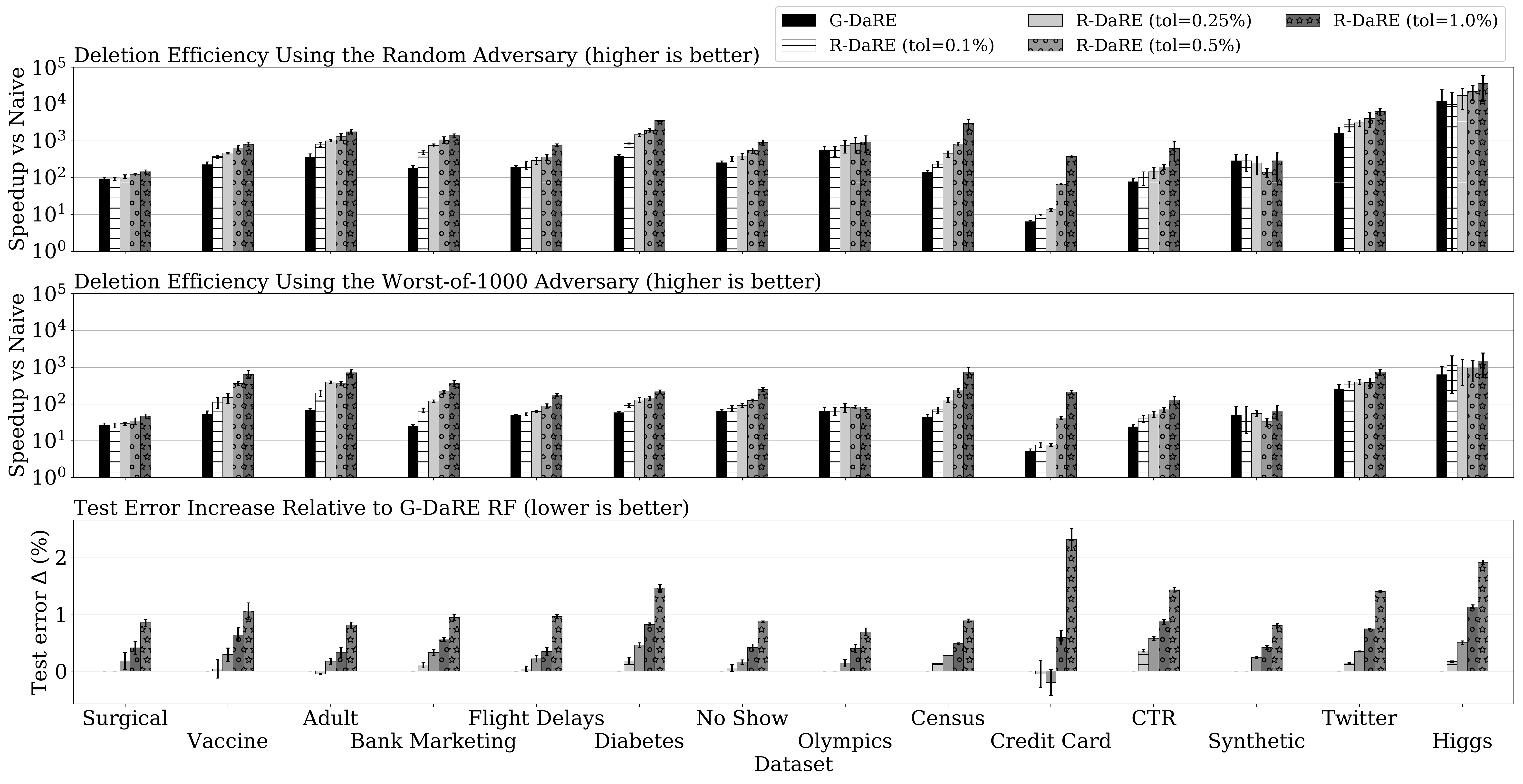}
\vskip -0.1in
\caption{Deletion efficiency of DaRE RF. \emph{Top \& Middle}: Number of instances deleted in the time it takes the naive retraining approach to delete one instance using the random and worst-of-1000 adversaries, respectively~(error bars represent standard deviation). \emph{Bottom}: The increase in test error when using R-DaRE RF relative to the predictive performance of G-DaRE RF~(error bars represent standard error).}
\label{fig:deletion}
\end{figure*}

\begin{figure*}
\centering
\includegraphics[width=1.0\textwidth]{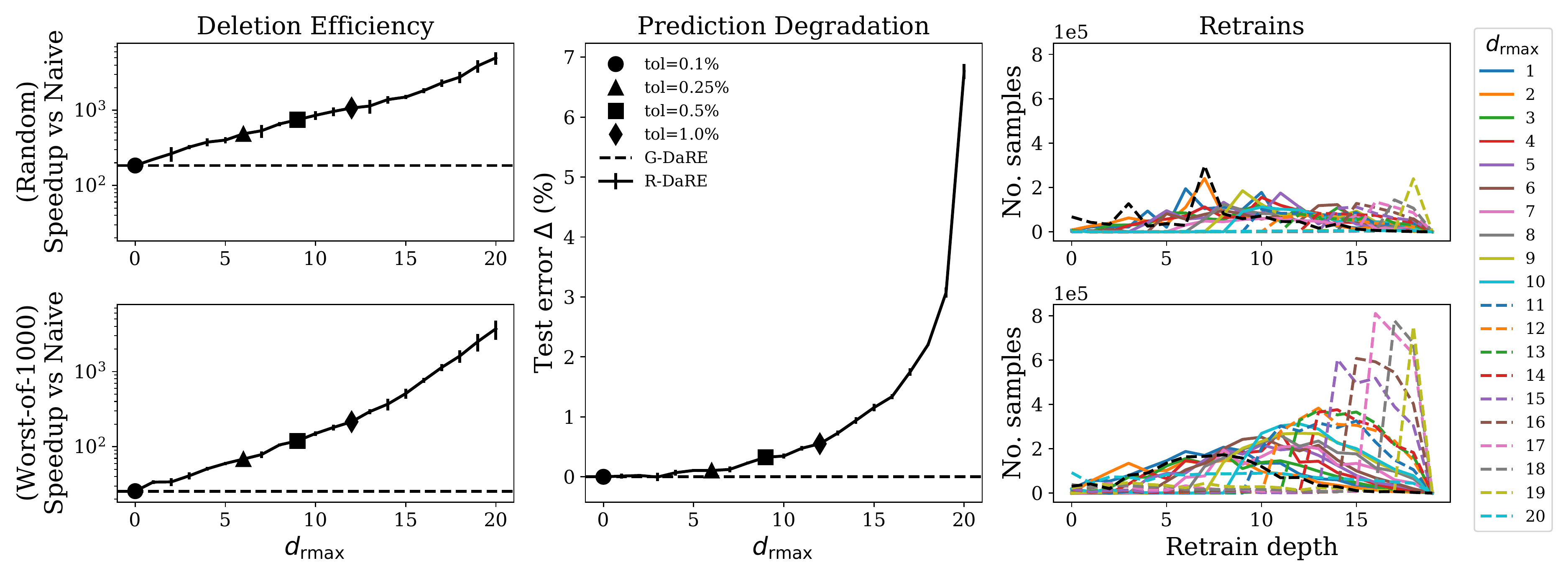}
\vskip -0.15in
\caption{Effect of increasing \dr on deletion efficiency~(left), predictive performance~(middle), and the cost of retraining~(right) using the random~(top) and worst-of-1000~(bottom) adversaries for the Bank Marketing dataset. The predictive performance is independent of the adversary, as performance is measured before any deletions occur. Error bars represent standard deviation and standard error for the left and middle plots, respectively. In short, we see that increasing \dr\ increases deletion efficiency but initially gradually degrades predictive performance. Similar analysis for other datasets are in the Appendix: \S\ref{appendix_subsec:d_rmax}.}
\label{fig:d_rmax}
\end{figure*}

We measure relative efficiency or speedup as the number of instances a DaRE model deletes in the time it takes the naive retraining approach to delete one instance~(i.e. retrain without that instance); the number of instances deleted gives us the speedup over the the naive approach.\footnote{System hardware specifications are in the Appendix:~\S\ref{appendix_sec:experiment_details}.} We also measure the predictive performance of R-DaRE RF prior to deletion and compare its predictive performance to that of G-DaRE RF. Each experiment is repeated five times.

We determine the order of deletions using two different adversaries: \emph{Random} and \emph{Worst-of-1000}. The random adversary selects training instances to be deleted uniformly at random, while the worst-of-1000 adversary selects each instance by first selecting 1,000 candidate instances uniformly at random, and then choosing the instance that results in the most retraining, as measured by the total number of instances assigned to all retrained nodes across all trees.

\subsection{Deletion Efficiency Results}

\paragraph{Random Adversary} We present the results of the deletion experiments using the random adversary in Figure~\ref{fig:deletion}~(top). We find that G-DaRE RF is usually at least two orders of magnitude faster than naive retraining, while R-DaRE RF is faster than G-DaRE RF to a varying degree depending on the dataset and error tolerance. R-DaRE RF is also able to maintain comparable predictive performance to G-DaRE RF, typically staying within a test error difference of 1\% depending on which tolerance is used to tune \dr~(Figure~\ref{fig:deletion}: bottom).


As an example of DaRE RF's utility,
naive retraining took 1.3 hours to delete a single instance for the Higgs dataset. R-DaRE RF ($tol=0.25$\% resulting in \dr$=3$) deleted over 17,000 instances in that time, an average of 0.283s per deletion, while the average test set error increased by only $0.5$\%. In this case, R-DaRE RF provides a speedup of over four orders of magnitude, providing a tractable solution for something previously intractable.



\paragraph{Worst-of-1000 Adversary} Against the more challenging worst-of-1000 adversary (Figure~\ref{fig:deletion}~(middle)), the speedup over naive deletion remains large, but is often an order of magnitude smaller. While R-DaRE models also decrease in efficiency, they maintain a significant advantage over G-DaRE RF, showing very similar trends of increased relative efficiency as when using the random adversary.

\begin{table}[t]
\centering
\caption{Summary of the deletion efficiency results. Specifically, the minimum, maximum, and geometric mean~(G. mean) of the speedup vs.\ the naive retraining method across all datasets.}
\vskip 0.15in
\label{tab:deletion_efficiency_summary}
\begin{tabular}{lrrr}
\toprule
\textbf{Model} & \textbf{Min.} & \textbf{Max.} & \textbf{G.\ Mean} \\ \midrule
\multicolumn{4}{l}{\textbf{Random Adversary}} \\
G-DaRE                & 6x   & 12,232x  & 257x   \\
R-DaRE (tol=0.1\%)    & 10x  &  9,735x  & 366x   \\
R-DaRE (tol=0.25\%)   & 13x  & 17,044x  & 494x   \\
R-DaRE (tol=0.5\%)    & 68x  & 22,011x  & 681x   \\
R-DaRE (tol=1.0\%)    & 145x & 35,856x  & 1,272x \\
\midrule
\multicolumn{4}{l}{\textbf{Worst-of-1000 Adversary}} \\
G-DaRE               & 5x  &   626x & 52x  \\
R-DaRE (tol=0.1\%)   & 8x  & 1,106x & 79x  \\
R-DaRE (tol=0.25\%)  & 8x  &   961x & 102x \\
R-DaRE (tol=0.5\%)   & 33x &   950x & 139x \\
R-DaRE (tol=1.0\%)   & 47x & 1,476x & 263x \\ \bottomrule
\end{tabular}
\end{table}

\paragraph{Summary} A summary of the deletion efficiency results is in Table~\ref{tab:deletion_efficiency_summary}. When instances to delete are chosen randomly, G-DaRE RF is more than 250x faster than naively retraining after every deletion (taking the geometric mean over the 14 datasets). By adding randomness, R-DaRE models achieve even larger speedups, from 360x to over 1,200x, depending on the predictive performance tolerance~(0.1\% to 1.0\%). The more sophisticated worst-of-1000 adversary can force more costly retraining. In this case, G-DaRE RF is more than 50x faster than naive retraining, and R-DaRE RF ranges from 80x to 260x depending on the tolerance.

\subsection{Effect of \dr and $k$ on Deletion Efficiency}

\begin{figure}[t]
\centering
\includegraphics[width=0.48\textwidth]{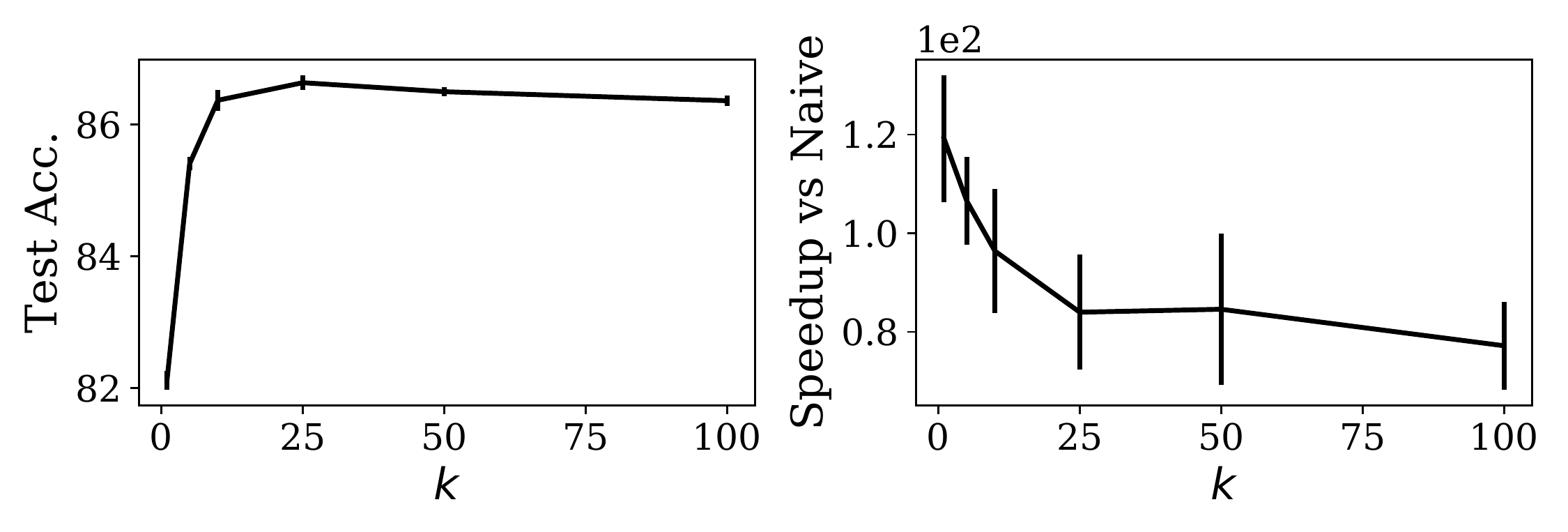}
\vskip -0.15in
\caption{Effect of increasing $k$ on predictive performance~(left) and deletion efficiency~(right) for the Surgical dataset using the random adversary; \dr\ is held fixed at 0. Error bars represent standard error and standard deviation for the left and right plots, respectively. Analysis for other datasets is in the Appendix: \S\ref{appendix_subsec:k}.}
\label{fig:k}
\end{figure}

Figure~\ref{fig:d_rmax} details the effect \dr~has on deletion efficiency under each adversary for the Bank Marketing dataset\footnote{Other datasets show similar trends; see the Appendix: \S\ref{appendix_subsec:d_rmax}.}.
As expected, we see that deletion efficiency increases as \dr~increases. Predictive performance degrades as \dr~increases, but initially degrades gracefully, maintaining a low increase in test error even as the top ten layers of each tree are replaced with random nodes~(+0.346\% test error).

Figure~\ref{fig:d_rmax} also shows the number of instances retrained at each depth, across all trees in the model. We immediately notice the increase in retraining cost when switching from the random~(top-right plot) to the worst-of-1000~(bottom-right plot) adversary, especially at larger depths. This matches our intuition since nodes deeper in the tree have fewer instances; each instance thus has a larger influence on the resulting split criterion over all attributes at a given node and increases the likelihood that a chosen attribute may change, resulting in more subtree retraining.

Figure~\ref{fig:k} shows the effect increasing $k$ has on predictive performance and deletion efficiency for the Surgical dataset\footnote{Other datasets show similar trends; see the Appendix: \S\ref{appendix_subsec:k}.}. In general, we find $k$ introduces a trade-off between predictive performance and deletion efficiency. However, our experiments show that modest values of $k$ can achieve competitive predictive performance while maintaining a high degree of deletion efficiency and incurring low storage costs.

\subsection{Space Overhead}

\begin{table*}[t]
\centering
\caption{Memory usage (in megabytes) for the training data, G-DARE RF, and an SKLearn RF trained using the same values of $T$ and $d_{max}$ as G-DARE RF. The total memory usage for the G-DARE RF model is broken down into: 1) the structure of the model needed for making predictions~(Structure); 2) the additional statistics stored at each decision node~(Decision Stats.); and 3) the additional statistics and training-instance pointers stored at each leaf node~(Leaf Stats.). The space overhead for G-DARE RF to enable efficient data deletion is measured as a ratio of the total memory usage of (data + G-DARE RF) to (data + SKLearn RF). Results are averaged over five runs and the standard error is shown in parentheses.}
\vskip 0.15in
\label{tab:space_overhead}
\begin{tabular}{lrrrrrrr}
\toprule
& & \multicolumn{4}{c}{\textbf{G-DARE RF}} & \\
\cmidrule(lr){3-6}
\textbf{Dataset} & \textbf{Data} & Structure &
Decision Stats. & Leaf Stats. &
\textbf{Total} & \textbf{SKLearn RF} &
\textbf{Overhead} \\ \midrule
Surgical           & 4     & 15  (0)    & 388    (\hphantom{00}1) & 14    (0)   & 417    (\hphantom{00}1)    & 31    (0) &  12.0x \\
Vaccine            & 16    & 18  (0)    & 426    (\hphantom{00}1) & 14    (0)   & 458    (\hphantom{00}2)    & 37    (0) & 8.9x \\
Adult              & 14    & 9   (0)    & 227    (\hphantom{00}1) & 16    (0)   & 252    (\hphantom{00}1)    & 18    (0) & 8.3x \\
Bank Marketing     & 8     & 23  (0)    & 455    (\hphantom{00}2) & 33    (0)   & 511    (\hphantom{00}2)    & 51    (0) & 8.8x \\
Flight Delays      & 207   & 37  (0)    & 3,030  (\hphantom{00}4) & 171   (0)   & 3,238  (\hphantom{00}5)    & 66    (0) & 12.6x \\
Diabetes           & 83    & 125 (0)    & 4,968  (\hphantom{0}12) & 199   (0)   & 5,292  (\hphantom{0}12)    & 257   (1) & 15.8x \\
No Show            & 35    & 91  (0)    & 2,511  (\hphantom{00}5) & 203   (0)   & 2,805  (\hphantom{00}5)    & 187   (1) & 12.8x \\
Olympics           & 663   & 27  (0)    & 3,196  (\hphantom{0}22) & 338   (0)   & 3,561  (\hphantom{0}23)    & 57    (0) & 5.9x \\
Census             & 326   & 33  (0)    & 1,737  (\hphantom{0}14) & 169   (0)   & 1,939  (\hphantom{0}14)    & 63    (0) & 5.8x \\
Credit Card        & 27    & 5   (0)    & 105    (\hphantom{00}1) & 457   (0)   & 567    (\hphantom{00}0)    & 7     (0) & 17.5x \\
CTR                & 45    & 6   (0)    & 485    (\hphantom{00}2) & 642   (0)   & 1,133  (\hphantom{00}0)    & 10    (0) & 21.4x \\
Twitter            & 48    & 186 (1)    & 2,450  (\hphantom{0}11) & 693   (0)   & 3,329  (\hphantom{0}12)    & 332   (0) & 8.9x \\
Synthetic          & 131   & 128 (1)    & 5,661  (\hphantom{0}36) & 357   (0)   & 6,146  (\hphantom{0}37)    & 114   (1) & 25.6x \\
Higgs              & 1,021 & 935 (4)    & 39,416 (168)            & 3,787 (1)   & 44,138 (173)               & 1,325 (9) & 19.3x \\
\bottomrule
\end{tabular}
\end{table*}

This section shows the space overhead of DARE forests by breaking the memory usage of G-DARE RF into three constituent parts: 1) the structure of the model that is needed for making predictions, 2) the additional statistics stored at each decision node, and 3) the additional statistics and training-instance pointers stored at each leaf node. Parts 2) and 3), plus the size of the data, constitute the space needed by G-DARE RF to enable efficient data removal.

Table~\ref{tab:space_overhead} shows the space overhead of G-DARE RF after training. We also show the training set size for each dataset, and the total memory usage of an SKLearn RF model using the same values for $T$ and $d_{max}$ as G-DARE RF.

As expected, decision-node statistics often make up the bulk of the space overhead for G-DARE RF; two exceptions are the Credit Card and CTR datasets, in which the size of the training-instance pointers outweigh the relatively low number of decision nodes~(an average of 238 and 726 per tree, respectively) for those models. The total memory usage of the G-DARE RF \emph{model} is 10-113x larger than that of the SKLearn RF model. However, since both approaches require the training data to enable deletions~(G-DARE RF may need to retrain subtrees; SKLearn RF needs to retrain using the naive approach), the relative overhead of G-DARE RF is the ratio of (data + G-DARE RF) to (data + SKLearn RF); this results in an overhead of 6--26x, quantifying the time/space trade-off for efficient data deletion.

\section{Related Work}\label{sec:related_work}

\paragraph{Exact Unlearning} There are a number of works that support exact unlearning of SVMs~\cite{cauwenberghs2001incremental,tveit2003incremental,duan2007decremental,romero2007incremental,karasuyama2009multiple,chen2019novel} in which the original goal was to accelerate leave-one-out cross-validation~\cite{shao1993linear}. More recently, \citet{cao2015towards} developed deletion mechanisms for several models that fall under the umbrella of non-adaptive SQ-learning~\cite{kearns1998efficient} in which data deletion is efficient and exact~(e.g. naive Bayes, item-item recommendation, etc.); \citet{schelter20} has also developed decremental update procedures for similar classes of models. \citet{ginart2019making} introduced a quantized variant of the $k$-means algorithm~\cite{lloyd1982least} called Q-$k$-means that supports exact data deletion. \citet{bourtoule2019machine} and \citet{aldaghri2020coded} propose training an ensemble of deep learning models on disjoint ``shards'' of a dataset, saving a snapshot of each model for every data point; the biggest drawbacks are the large storage costs, applicability only to \emph{iterative} learning algorithms, and the significant degradation of predictive performance. \citet{schelter2021hedgecut} enable efficient data removal for extremely randomized trees~(ERTs)~\cite{geurts2006extremely} without needing to save the training data by precomputing alternative subtrees for splits sensitive to deletions; aside from only being applicable to ERTs, they assume a very small percentage of instances can be deleted.

\paragraph{Approximate Unlearning}
In contrast to exact unlearning, a promising definition of approximate unlearning~(a.k.a statistical unlearning) guarantees $\forall \mathcal{S} \subseteq \mathcal{H}, \D, (x,y) \in \D: e^{-\epsilon} \leq P(\mathcal{R}(\mathcal{A}(\D), \D, (x, y)) \in S) \ / \ P(\mathcal{A}(\D \setminus (x, y)) \in S) \leq e^{\epsilon}$~($\epsilon$-certified removal: \citet{guo2020certified}, Eq. 1). \citet{golatkar_2020_CVPR,golatkar2020forgetting} propose a scrubbing mechanism for deep neural networks that does not require any retraining; however, the computational complexity of their approach is currently quite high. \citet{guo2020certified}, \citet{izzo2020approximate}, and \citet{wu2020deltagrad} propose different removal mechanisms for linear and logistic regression models that can be applied to the last fully connected layer of a deep neural network. \citet{golatkar2020mixed} perform unlearning on a linear approximation of large-scale vision networks in a mixed-privacy setting. \citet{fu2021bayesian} propose an unlearning procedure for models in a Bayesian setting using variational inference.

\paragraph{Mitigation}
Although not specifically designed as unlearning techniques, the following works propose different mechanisms for mitigating the impact of noisy, poisoned, or non-private training data. \citet{baumhauer2020machine} propose an output filtering technique that prevents private data from being leaked; however, their approach does not update the model itself, potentially leaking information if the model were still accessible. \citet{wang2019neural} and \citet{du2019lifelong} fine-tune their models on corrected versions of poisoned or corrupted training instances to mitigate backdoor attacks~\cite{gubadnets} on image classifiers and anomaly detectors, respectively. Although both approaches show promising empirical performance, they provide no guarantees about the extent to which these problematic training instances are removed from the model~\cite{sommer2020towards}. \citet{tople2019analyzing} analyze privacy leakage in language model snapshots before and after they are updated.

\paragraph{Differential Privacy}

Differential privacy~(DP)~\cite{dwork2006differential,JMLR:v12:chaudhuri11a,abadi} is a sufficient condition for approximate unlearning (in the case of a single deletion, sequential deletions may require using group DP~\cite{dwork2014algorithmic}), but it is an unnecessary and overly strict one since machine unlearning does not require instances to be private~\cite{guo2020certified}. Furthermore, differentially-private random forest models often suffer from poor predictive performance~\cite{fletcher2015differentially,fletcher2019decision}; this is because the privacy budget~(typically denoted $\epsilon$ or $\beta$) must be split among all the trees in the forest, and among the different layers in each tree. This typically results in a meaningless privacy budget~(e.g. $\epsilon > 10$)~\cite{fletcher2019decision}, a relaxed definition of DP~\cite{rana2015differentially}, extremely randomized trees~\cite{geurts2006extremely,fletcher2017differentially}, or very small forests~(e.g. $T=10$)~\cite{consul2020differentially}.

\section{Discussion}

Since data deletions in DaRE models are exact, membership inference attacks~\cite{yeom2018privacy,carlini2018secret} are guaranteed to be unsuccessful for instances deleted from the model. DaRE models also reduce the need for deletion verification methods~\cite{shintre2019verifying,sommer2020towards}. However, one must be aware that DaRE models (as well as any unlearning method) can leak which instances are deleted if an adversary has access to the model before \emph{and} after the deletion~\cite{chen2020machine}. Although privacy is a strong motivator for this work, there are a number of other useful applications for DaRE forests.

\paragraph{Instance-Based Interpretability}
A popular form of interpretability looks at how much each training instance contributes to a given prediction~\cite{koh2017understanding,yeh2018representer,pmlr-v80-sharchilev18a,pruthi2020estimating,chen2021hydra}. The naive approach to this task involves leave-one-out retraining for every training instance in order to analyze the effect each instance has on a target prediction, but this is typically intractable for most machine learning models and datasets. DaRE models can more efficiently compute the same training-instance attributions as the naive approach, making leave-one-out retraining a potentially viable option for generating instance-attribution explanations for random forest models.

\paragraph{Dataset Cleaning}
Aside from removing user data for privacy reasons, one may also wish to efficiently remove outliers~\cite{rahmani2019outlier,yihe2019outlier} or training instances that are noisy, corrupted, or poisoned~\cite{mozaffari2014systematic,steinhardt2017certified}.

\paragraph{Continual Learning}
Our methods can also be used to \emph{add} data to a random forest model, allowing for continuous updating as data is added and removed. This makes them well-suited to continual learning settings with streaming data~\cite{chrysakis2020online,knoblauch2020optimal}. However, the hyperparameters may need to be periodicially retuned as the size or distribution of the data shifts from adding and/or deleting more and more instances.

\paragraph{Eco-Friendly Machine Learning}
Finally, this line of research promotes a more economically and environmentally sustainable approach to building learning systems; if a model can be continuously updated only as necessary and avoid frequent retraining, significant time and computational resources can be spared~\cite{gupta2020secure}.

\section{Conclusion}

In this work, we introduced DaRE RF, a random forest variant that supports efficient model updates in response to repeated deletions of training instances. We find that, on average, DaRE models are 2-3 orders of magnitude faster than the naive retraining approach with no loss in accuracy, and additional efficiency can be achieved if slightly worse predictive performance is tolerated.

For future work, there are many exciting opportunities and applications of DaRE forests, from maintaining user privacy to building interpretable models to cleaning data, all without retraining from scratch. One could even investigate the possibility of extending DaRE forests to boosted trees~\cite{chen2016xgboost,ke2017lightgbm,prokhorenkova2018catboost}. At its best, DaRE RF was more than four orders of magnitude faster than naive retraining, so it has the potential to enable new applications of model updating that were previously intractable.

\section*{Acknowledgments}

We would like to thank Zayd Hammoudeh for useful discussions and feedback and the reviewers for their constructive comments that improved this paper. This work was supported by a grant from the Air Force Research Laboratory and the Defense Advanced Research Projects Agency (DARPA) --- agreement number FA8750-16-C-0166, subcontract K001892-00-S05. This work benefited from access to the University of Oregon high performance computer, Talapas.


\bibliography{references}
\bibliographystyle{icml2021}

\onecolumn
\appendix

\section{Algorithmic Details}\label{appendix_sec:algorithmic_details}

\subsection{Exact Deletion: Proof of Theorem~\ref{theorem:deletion_correctness}}

We use the following Lemma to help prove the theorem of exact deletion for DaRE forests.

\begin{lemma}
\label{lemma:set_deletion}
The probability of selecting a valid set of thresholds $S$ from a dataset $D$ and then subsequently resampling any invalidated thresholds after the deletion of $(x, y) \in D$ is equivalent to the probability of selecting $S$ from an updated dataset $D \setminus (x,y)$.
\end{lemma}

\begin{proof}
The probability of choosing a valid set of thresholds $S$ from $D \setminus (x,y)$ is $P^{A}(S) = 1 / \binom{n-m}{k}$ in which $n$ is the number of valid thresholds before the deletion, $k$ is the number of thresholds to sample from the set of valid thresholds, and $m$ is the number of thresholds that become invalid due to the deletion of $(x,y)$. The probability of ending up with thresholds $S$ by first choosing some set $S^{*}$ and then resampling any thresholds invalidated by removing $(x,y)$ is:
\begin{align*}
P^{B+R}(S) = \frac{1}{\binom{n}{k}} \sum_{i=0}^{m} \frac{\binom{m}{i}\binom{k}{i}}{\binom{n-k-(m-i)}{i}},
\end{align*}
in which $\binom{n}{k}$ is the number of valid threshold sets for $D$; $S^*$ may have up to $m$ invalid thresholds, thus $\binom{m}{i}\binom{k}{i}$ is the number of ways $i$ invalid thresholds out of $k$ chosen thresholds could be resampled from the set of $m$ invalid thresholds; and $\binom{n-k-(m-i)}{i}$ is the number of valid threshold sets that can be resampled to, starting at a set with $i$ invalid thresholds.

In the simplest case, $m=0$ and no thresholds are invalidated, so the probability of choosing $S$ in the updated dataset and original dataset are identical: $1 / \binom{n}{k} = 1 / \binom{n-0}{k}$.
In the next simplest case, $m=1$ and only a single threshold is invalidated. Thus, we could arrive at $S$ by first sampling it with the original dataset (probability $1 / \binom{n}{k}$) or by first sampling one of the $k$ sets that includes the invalidated threshold and is otherwise identical to $S$, followed by resampling that threshold from the remaining $(n-1)-(k-1)$ valid and unselected thresholds to obtain $S$.

Thus, the total probability (for $m=1$) is:
\begin{align*}
P^{B+R}(S) &= \frac{1}{\binom{n}{k}} \left(1 + \frac{k}{n-k}\right) \\
&= \frac{1}{\binom{n}{k}} \left(\frac{n}{n-k}\right) \\
&= \frac{k!\:(n-k)!\:n}{n!\:(n-k)} \\
&= \frac{k!(n-1-k)!}{(n-1)!} \\
&= \frac{1}{\binom{n-1}{k}} \\
&= P^{A}(S)
\end{align*}
For $m>1$, we can reduce it to the $m=1$ case by viewing it as a sequence of invalidating one threshold at a time. After invalidating one of the thresholds, the probability remains uniform, so by induction it continues to remain uniform after a second deletion, or a third, or any number.
\end{proof}

\begin{theorem*}
Data deletion for DaRE forests is \emph{exact}~(see Eq.~\ref{eq:exact_unlearning}), meaning that removing instances from a DaRE model yields exactly the same model as retraining from scratch on updated data.
\end{theorem*}

\begin{proof}
Exact unlearning is defined as having the same probability distribution over models by deletion as by retraining~(Def. (1)). For discrete attributes, the node statistics used in model updating are precisely those used for learning the initial structure, so as the statistics are updated, the structure is updated to match what would be learned from scratch (in distribution).

For continuous attributes, we first discretize by uniformly sampling $k$ thresholds from the set of all valid thresholds for that attribute. As instances are removed, if one of the sampled thresholds becomes invalid, then those thresholds are resampled to obtain a set of valid thresholds. Lemma~\ref{lemma:set_deletion} shows the resulting probability of each set of valid thresholds remains uniform, identical to what it would be if the model were retrained from scratch.

The same logic and lemma also applies for attributes. If a deletion causes one or more attributes to become invalid~(i.e. no more valid thresholds to sample), then those attributes are resampled to obtain a set of valid attributes, with all sets of valid attributes being equally likely.

Since each decision node in the tree operates on its own partition of the data $D$, then updating all relevant decision nodes and leaf nodes results in the entire tree being updated to match the updated dataset. The extension to the forest follows since all trees are independent; thus, the probability of a DaRE forest after removing instances is the same as retraining the model from scratch on updated data.

\end{proof}

\subsection{Training Complexity: Proof of Theorem~\ref{theorem:training_runtime}}

\begin{theorem*}
Given $n = |\D|$, $T$, $d_{\max}$, and $\Tilde{p}$, the time complexity to train a DaRE forest is $\mathcal{O}(T\, \Tilde{p}\,n\,d_{\max})$.
\end{theorem*}

\begin{proof}
When training a DaRE tree,  we begin by choosing a split for the root by iterating through all $n$ training instances and scoring $\tilde{p}$ randomly selected attributes. Generalizing this to nodes at other depths, there are (at most) $2^d$ nodes at depth $d$, and each of the $n$ training instances is assigned to one of these nodes. Choosing all splits at depth $d$ thus requires a total time of $\bigO(\tilde{p}\, n)$ across all depth-$d$ nodes, since we again process every training instance when finding the best split for each node. Summing over all depths, the total time is $\bigO(\tilde{p}\, n\, d_{\max})$ to train a single DaRE tree or $\bigO(T\, \tilde{p}\, n\, d_{\max})$ to train a forest of $T$ trees.
\end{proof}

\subsection{Training Complexity: Proof of Theorem~\ref{theorem:deletion_runtime}}

\begin{theorem*}
Given $d_{\max}$, $\Tilde{p}$, and $k$, the time complexity to delete a single instance $(x,y) \in \D$ from a DaRE tree is $\mathcal{O}(\Tilde{p}\,k\,d_{\max})$,
if the tree structure is unchanged and the attribute thresholds remain valid. 
If a node with $|D|$ instances has an invalid attribute threshold, then the additional time to choose new thresholds is $\mathcal{O}(|D| \log |D|)$.
If a node with $|D|$ instances at level $d$ needs to be retrained, then the additional retraining time is $\mathcal{O}(\Tilde{p}\,|D|\,(d_{\max}-d)))$.
\end{theorem*}


\begin{proof}
Deleting an instance from a DaRE tree (Alg.~\ref{alg:update}) requires traversing the tree from the root to the a leaf, updating node statistics, retraining a subtree (if necessary), and removing the instance from the tree's set of instances. Since there are $\Tilde{p}$ candidate attributes at each node, subtracting the influence of $(x,y)$ from the node statistics and checking for invalid attribute thresholds requires $\bigO(\Tilde{p})$ time. Recomputing the score for each attribute-threshold pair requires $\mathcal{O}(\Tilde{p}\,k)$, since we have the necessary statistics and computing the Gini index can be done in constant time for each pair. Across all depths up to $d_{\max}$, this is a total time of $\mathcal{O}(\Tilde{p}\, k\, d_{\max})$.

Choosing new thresholds requires making a list of all attribute values at a node, along with the associated labels. This can be done by traversing the subtree rooted at the node, visiting each leaf and collecting the attribute values from the instances at that leaf. Let $|D|$ be the total number of these instances. Since the number of leaves is bounded by the number of instances, traversing the subtree can be done in $\mathcal{O}(|D|)$ time, plus $\mathcal{O}(|D| \log |D|)$ time to sort the values. The remaining work of making a list of valid thresholds, randomly choosing $k$ thresholds, and computing statistics for these $k$ thresholds can all be done in $\mathcal{O}(|D|)$, since each requires (at most) a single pass through all $|D|$ instances. Thus, the total time is $\mathcal{O}(|D| \log |D|)$.

If the best attribute-threshold pair at a node has changed, then the subtree must be retrained. Let $|D|$ be the number of instances at the node and $d$ be its depth. The time for retraining a subtree is identical to the time for retraining a DaRE tree, except that the number of instances is $|D|$ and the maximum depth (relative to this node) is $(d_{\max}-d)$. Thus, the total time is $\mathcal{O}(\Tilde{p}\, (d_{\max}-d)\, |D|)$.
\end{proof}

\subsection{Space Complexity: Proof of Theorem~\ref{theorem:space_complexity}}

\begin{theorem*}
Given $\D$, $d_{\max}$, $k$, $T$, and $\Tilde{p}$, the space complexity of a DaRE forest is $\mathcal{O}(k\, \Tilde{p}\, 2^{d_{\max}}\, T + n\,T$).
\end{theorem*}

\begin{proof}
The space complexity of a DaRE tree with a single decision node is $\mathcal{O}(k\, \Tilde{p} + n)$ since we need to store a constant $\mathcal{O}(1)$ amount of metadata for $k$ thresholds times $\Tilde{p}$ attributes as well as $n$ pointers~(one for each training instance) partitioned across the leaves in the tree. For a single DaRE tree with multiple decision nodes, we need to multiply the first term in the previous result by $2^{d_{\max}}$ since there may be $2^{d_{\max}}$ decision nodes in a single DaRE tree; the second term remains the same as the training instances are still partitioned across all leaves in the tree. Thus, the space complexity of a single DaRE tree is $\mathcal{O}(k\, \Tilde{p}\, 2^{d_{\max}}\, + n)$. For a DaRE forest, we need to multiply this result by $T$; thus, the space complexity of a DaRE forest is $\mathcal{O}(k\, \Tilde{p}\, 2^{d_{\max}}\, T + n\,T$).
\end{proof}

Assuming that we have at least one training instance assigned to each leaf, the number of leaves in each tree is at most $n$, and thus the total number of nodes per tree is at most $2n-1$. This gives us an alternate bound of $\bigO(k \Tilde{p} n T)$, which is proportional to the size of the training data times the number of thresholds and trees in the forest (in the worst case).


        

\subsection{Complexity of Slightly-Less-Naive Retraining}

The complexity of naive retraining is the same as training a DaRE forest from scratch, $\mathcal{O}(T\,\Tilde{p}\,n\,d_{\max})$, where $n = |\D|$.

A slightly smarter approach is to retrain only the portion of each tree that depends on the deleted node. For example, if the best split at the root of the tree after deleting an instance is the same as it was before, then the data will be partitioned between its two children the same way as before. One part of this partition never contained the deleted instance, and that fraction of the tree is unchanged. The other part of the partition has been changed by this deletion, so we must recurse, but only in that half of the tree.

This is potentially more efficient, but the efficiency gains are still bounded relative to the naive retraining approach. Choosing the split at the root still requires iterating through all training instances to compute statistics for each attribute (and each split of each continuous attribute), for a total time of $\bigO(T\, \tilde{p}\, n)$ across all $T$ trees. 

Since the total time for retraining is at most $\bigO(T\, \tilde{p}\, n\, d_{\max})$, the gain from this optimized approach is at \emph{most} a factor of $d_{\max}$ (10-20 in our experiments). This is ignoring the cost of scoring splits at the lower levels in the tree, or retraining the lower levels of the tree (as is often required after data deletion). Thus, the gain will be smaller in practice.

Therefore, a slightly-less-naive approach to retraining random forests could improve over the naive approach, but would still be substantially slower than our methods, which achieve speedups of several orders of magnitude (see Figure~\ref{fig:deletion} and Table~\ref{tab:deletion_efficiency_summary}).

\subsection{Node Statistics}\label{appendix_subsec:node_statistics}

A DaRE tree may consist of three types of nodes: greedy decision nodes, random decision nodes, and leaf nodes. Each stores a constant amount of metadata to enable efficient updates. In addition to the following type-specific statistics, each node stores $|D|$ and $|D_{\cdot, 1}|$ the number of instances and the number of positive instances at that node.

\begin{itemize}
    \item Greedy decision nodes: For each threshold for an attribute, we store $|D_{l}|$, $|D_{l,1}|$. This is a sufficient set of statistics needed to recompute the Gini index~(Eq.~\ref{eq:gini_index}) or entropy~(Eq.~\ref{eq:entropy}) split criterion scores. Since a threshold in a greedy decision node is the midpoint between two adjacent attribute values, we also keep track of how many positive instances and the total number of instances are in each attribute value set; by updating this information, a DaRE tree can sample a new threshold when one is no longer valid.
    
    \item Random decision nodes: After selecting a random attribute, and then a random threshold within that attribute's min.\ and max.\ value range, we only store $|\D_{l}|$ and $|\D_{r}|$. Updating these statistics informs the DaRE tree when the threshold value is no longer within the min.\ and max.\ value range of that attribute. At that point, the random decision node is retrained.
    
    \item Leaf nodes: For each leaf, we store pointers to the training instances that traversed to that leaf. This enables the DaRE tree to collect these training instances when needing to retrain any ancestor decision nodes higher in the tree.
\end{itemize}

\subsection{Batch Deletion}

Batch deletion is almost the same as deleting one instance, except we may need to recurse down multiple branches of each tree to find all relevant instances to delete, and we only retrain a given node~(at most) once, rather than~(up to) once for each instance deleted. This will naturally be more efficient, but waiting for a large batch may not be possible.

\newpage

\subsection{Pseudocode}\label{appendix_subsec:pseduocode}

Algorithm~\ref{alg:appendix} provides detailed pseudocode for training a DARE tree and deleting an instance from a trained DARE tree. For even more code details, please check out our code repository at \url{https://github.com/jjbrophy47/dare\_rf}.

\begin{algorithm}[ht!]
\small
\caption{Pseudocode for building a DaRE tree / subtree, and deleting an instance from a DARE tree (unabridged).}
\label{alg:appendix}
\begin{minipage}{0.46\textwidth}
\begin{algorithmic}[1]
    \STATE{\textbf{\textsc{Train}}(data $D$, depth $d$):}
    \begin{ALC@g}
      \IF{\text{stopping criteria reached}}
        \STATE $node$ $\gets$ \textsc{LeafNode}($D$)
      \ELSE
        \IF{$d <$ \dr}
          \STATE{$node$ $\gets$ \textsc{RandomNode}($D$)}
        \ELSE
          \STATE{$node$ $\gets$ \textsc{GreedyNode}($D$)}
        \ENDIF
        \STATE{$D.\ell, D.r \gets$ split on selected threshold($node$, $D$)}
        \STATE{$node.\ell, r$ $\gets$ \textsc{Train}($\D_{\ell}$, $d+1$), \textsc{Train}($\D_r$, $d+1$)}
      \ENDIF
    \STATE {\bfseries return} $node$
    \end{ALC@g}

    \item[]

    \STATE{\textbf{\textsc{Delete}}($node$, depth $d$, instance to remove $(x, y)$):}
    \begin{ALC@g}
      \STATE{$node \Lleftarrow$ \text{update pos. and total count}($node$, $(x, y)$)}
      \IF{$node$ is a \textsc{LeafNode}}
        \STATE{$node \Lleftarrow$ \text{remove $(x, y)$ from leaf instances}}
        \STATE{$node \Lleftarrow$ \text{recompute leaf value}($node$)}
        \STATE{\text{remove $(x, y)$ from database}}
      \ELSE
        \STATE{$node \Lleftarrow$ \text{update decision statistics}($node$, $(x, y)$)}
        \IF{$node$ is a \textsc{RandomNode}}
            \STATE{$node \gets$ \textsc{RandomNodeDelete}($node$, $d$, $(x, y)$)}
        \ELSE
            \STATE{$node \gets$ \textsc{GreedyNodeDelete}($node$, $d$, $(x, y)$)}
        \ENDIF
        \STATE{$a, v \gets$ \text{$node$.selected attribute, threshold}}
        \IF{no retraining occurred}
            \IF{$x_{\cdot, a} \leq v$}
                \STATE{\textsc{Delete}($node.\ell$, $d+1$, $(x, y)$)}
            \ELSE
                \STATE{\textsc{Delete}($node.r$, $d+1$, $(x, y)$)}
            \ENDIF
        \ENDIF
      \ENDIF
    \STATE {\bfseries return $node$}
    \end{ALC@g}

    \item[]

    \STATE{\textbf{\textsc{GreedyNodeDelete}}($node$, depth $d$, inst. $(x, y)$):}
    \begin{ALC@g}
        \STATE{$A \gets node$.sampled attributes}
        \STATE{$\bar{A} \gets$ \text{get invalid attributes}($A$)}
        \IF{$|\bar{A}| > 0$}
            \STATE{$D$ $\gets$ get data from leaves($node$) $\setminus$ $(x, y)$}
            \hfill{$\vartriangleright$~\text{Cache}}
            \STATE{$A^* \gets$ \text{resample invalid attributes}($\bar{A}$, $D$)}
            \STATE{$A \gets A \setminus \bar{A} \cup A^*$}
        \ENDIF
        \FOR{$a \in A$}
            \STATE{$V \gets a$.sampled valid thresholds}
            \STATE{$\bar{V} \gets$ \text{get invalid thresholds}($V$)}
            \IF{$|\bar{V}| > 0$}
                \STATE{$D$ $\gets$ \text{get data from leaves}($node$) $\setminus (x, y)$}
                \hfill{$\vartriangleright$~\text{Cache}}
                \STATE{$V^* \gets$ \text{resample invalid thresholds}($\bar{V}$, $D$, $a$)}
                \STATE{$V \gets V \setminus \bar{V} \cup V^*$}
            \ENDIF
        \ENDFOR
        \STATE{$scores$ $\gets$ \text{recompute split scores}($node$)}
        \STATE{$node \Lleftarrow$ \text{select optimal split}($scores$)}
        \IF{optimal split has changed}
            \STATE{$a, v \gets$ \text{$node$.selected attribute, threshold}}
            \STATE{$D_{\ell}, D_r \gets$ \text{split data on new threshold}($D$, $a$, $v$)}
            \STATE{$node.\ell \gets$ \textsc{Train}($\D_{\ell}$, $d+1$)} \hfill{$\vartriangleright$~\text{Retrain left}}
            \STATE{$node.r \gets$ \textsc{Train}($\D_{r}$, $d+1$)} \hfill{$\vartriangleright$~\text{Retrain right}}
        \ENDIF
    \STATE {\bfseries return} $node$
    \end{ALC@g}

\end{algorithmic}
\end{minipage}
\hfill
\begin{minipage}{0.46\textwidth}
\begin{algorithmic}[1]
    \STATE{\textbf{\textsc{LeafNode}}(data $D$):}
    \begin{ALC@g}
    \STATE{$node \gets$ \textsc{Node()}}
    \STATE{$node \gets$ \textsc{SaveNodeStats}($node$, $D$)}
    \STATE{$node \Lleftarrow$\footnotemark \text{compute leaf value}($node$)}
    \STATE{$node \Lleftarrow$ \text{save leaf instances}($node$, $D$)}
    \STATE{\textbf{return} $node$}
    \end{ALC@g}
    
    \item[]
    
    \STATE{\textbf{\textsc{RandomNode}}(data $D$):}
    \begin{ALC@g}
    \STATE{$node \gets$ \textsc{Node()}}
    \STATE{$node \gets$ \textsc{SaveNodeStats}($node$, $D$)}
    \STATE{$a \gets$ \text{randomly sample attribute}($D$)}
    \STATE{$v \gets$ \text{randomly sample threshold $\in [a_{\min}$,
    $a_{\max})$}}
    \STATE{$node \gets$ \textsc{SaveThreshStats}($node$, $D$, $a$, $v$)}
    \STATE{\textbf{return} node}
    \end{ALC@g}
    
    \item[]
    
    \STATE{\textbf{\textsc{GreedyNode}}(data $D$):}
    \begin{ALC@g}
    \STATE{$node \gets$ \textsc{Node()}}
    \STATE{$node \gets$ \textsc{SaveNodeStats}($node$, $D$)}
    \STATE{$node \Lleftarrow$ \text{randomly sample $\tilde{p}$ attributes}($node$, $D$)}
    \FOR{$a \in$ \text{$node$.sampled attributes}}
        \STATE{$C$ $\gets$ \text{get valid thresholds}($D$, $a$)}
        \STATE{$V$ $\gets$ \text{randomly sample $k$ valid thresholds}($C$)}
        \FOR{$v \in V$}
            \STATE{$node \gets$ \textsc{SaveThreshStats}($node$, $D$ $a$, $v$)}
        \ENDFOR
    \ENDFOR
    \STATE{$scores \gets$ \text{compute split scores}($node$)}
    \STATE{$node \Lleftarrow$ \text{select optimal split}($scores$)}
    \STATE {\bfseries return} $node$
    \end{ALC@g}
    
    \item[]

    \STATE{\textbf{\textsc{SaveNodeStats}}($node$, data $D$):}
    \begin{ALC@g}
    \STATE{$node \Lleftarrow$ instance count($D$)}
    \STATE{$node \Lleftarrow$ positive instance count($D$)}
    \STATE {\bfseries return} $node$
    \end{ALC@g}
    
    \item[]
    
    \STATE{\textbf{\textsc{SaveThreshStats}}($node$, data $D$, attr. $a$, thresh. $v$):}
    \begin{ALC@g}
    \STATE{$node \Lleftarrow$ \text{left branch instance count}($D$, $a$, $v$)}
    \IF{$node$ is a \textsc{GreedyNode}}
        \STATE{$node \Lleftarrow$ left branch pos. instance count($D$, $a$, $v$)}
        \STATE{$node \Lleftarrow$ left / right adj. feature val.($D$, $a$, $v$)}
        \STATE{$node \Lleftarrow$ left / right adj. val. set count($D$, $a$, $v$)}
        \STATE{$node \Lleftarrow$ left / right adj. val. set pos. count($D$, $a$, $v$)}
    \ENDIF
    \STATE {\bfseries return} $node$
    \end{ALC@g}
    
    \item[]
    
    \STATE{\textbf{\textsc{RandomNodeDelete}}($node$, depth $d$, inst. $(x, y)$:}
    \begin{ALC@g}
    \IF{selected threshold is invalid}
        \STATE{$D$ $\gets$ get data from leaves($node$) $\setminus$ $(x, y)$}
        \IF{\text{selected attribute~($a$) is still valid}}
            \STATE{$v \gets$ \text{resample threshold $\in [a_{min}, a_{max})$}}
            \STATE{$D_{\ell}, D_r \gets$ \text{split data on new threshold}($D$, $a$, $v$)}
            \STATE{$node.\ell \gets$ \textsc{Train}($\D_{\ell}$, $d+1$)} \hfill{$\vartriangleright$~\text{Retrain left}}
            \STATE{$node.r \gets$ \textsc{Train}($\D_{r}$, $d+1$)} \hfill{$\vartriangleright$~\text{Retrain right}}
        \ELSE
            \STATE{$node \gets$ \textsc{Train}($\D$, $d$)}
            \hfill{$\vartriangleright$~\text{Retrain entire subtree}}
        \ENDIF
      \ENDIF
    \STATE {\bfseries return} $node$
    \end{ALC@g}

\end{algorithmic}
\end{minipage}
\end{algorithm}
\footnotetext{Each node type has several types of counts and statistics to maintain. Rather than name and update each count separately, we use ``$node \Lleftarrow \ldots$'' to denote updates to a node or its data. Details about node statistics for each node type are in \S\ref{appendix_subsec:node_statistics}.}

\newpage

\section{Implementation and Experiment Details}\label{appendix_sec:experiment_details}

Experiments are run on an Intel(R) Xeon(R) CPU E5-2690 v4 @ 2.6GHz with 70GB of RAM. No parallelization is used when building the independent decision trees. DaRE RF is implemented in the C programming language via Cython, a Python package allowing the development of C extensions. Experiments are run using Python 3.7. Source code for DaRE RF and all experiments is available at \url{https://github.com/jjbrophy47/dare_rf}.

\subsection{Datasets}\label{sec:datasets_appendix}

\begin{itemize}
    \item Surgical~\cite{surgical} consists of 14,635 medical patient surgeries (3,690 positive cases), characterized by 25 attributes; the goal is to predict whether or not a patient had a complication from their surgery.
    
    \item Vaccine~\cite{bull2016harnessing,vaccine} consists of 26,707 survey responses collected between October 2009 and June 2010 asking people a range of 36 behavioral and personal questions, and ultimately asking whether or not they got an H1N1 and/or seasonal flu vaccine. Our aim is to predict whether or not a person received a seasonal flu vaccine.

    \item Adult \cite{Dua:2019} contains 48,842 instances (11,687 positive) of 14 demographic attributes to determine if a person's personal income level is more than \$50K per year.

    \item Bank Marketing \cite{moro2014data,Dua:2019} consists of 41,188 marketing phone calls (4,640 positive) from a Portuguese banking institution. There are 20 attributes, and the aim is to figure out if a client will subscribe.

    \item Flight Delays~\cite{flight_delays} consists of 100,000 actual arrival and departure times of flights by certified U.S. air carriers; the data was collected by the Bureau of Transportation Statistics' (BTS) Office of Airline Information. The data contains 8 attributes and 19,044 delays. The task is to predict if a flight will be delayed.

    \item Diabetes \cite{strack2014impact,Dua:2019} consists of 101,766 instances of patient and hospital readmission outcomes (46,902 readmitted) characterized by 55 attributes.
    
    \item No Show~\cite{no_show} contains 110,527 instances of patient attendances for doctors' appointments  (22,319 no shows) characterized by 14 attributes. The aim is to predict whether or not a patient shows up to their doctors' appointment.
    
    \item Olympics~\cite{olympics} contains 206,165 Olympic events over 120 years of Olympic history. Each event contains information about the athlete, their country, which Olympics the event took place, the sport, and what type of medal the athlete received. The aim is to predict whether or not an athlete received a medal for each event they participated in.

    \item Census~\cite{Dua:2019} contains 40 demographic and employment attributes on 299,285 people in the United States; the survey was conducted by the U.S. Census Bureau. The goal is to predict if a person's income level is more than \$50K.
    
    \item Credit Card~\cite{creditcard} contains 284,807 credit card transactions in September 2013 by European cardholders. The transactions took place over two days and contains 492 fraudulent charges~(0.172\% of all charges). There are 28 principal components resulting from PCA on the original dataset, and two additional fetures: `time' and `amount'. The aim is to predict whether a charge is fraudulent or not.
    
    \item Click-Through Rate~(CTR)~\cite{ctr} contains the first 1,000,000 instances of the Criteo 1TB Click Logs dataset, in which each row represents an ad that was displayed and whether or not it had been clicked on (29,040 ads clicked). The dataset contains 13 numeric attributes and 26 categorical attributes. However, due to the extremely large number of values for the categorical attributes, we restrict our use of the dataset to the 13 numeric attributes. The aim is to predict whether or not an ad is clicked on.
    
    \item Twitter uses the first 1,000,000 tweets (169,471 spam) of the HSpam14 dataset~\cite{sedhai2015hspam14}. Each instance contains the tweet ID and label. After retrieving the text and user ID for each tweet, we derive the following attributes: no. chars, no. hashtags, no. mentions, no. links, no. retweets, no. unicode chars., and no. messages per user. The aim is to predict whether a tweet is spam or not.
    
    \item Synthetic~\cite{scikit-learn} contains 1,000,000 instances normally distributed about the vertices of a 5-dimensional hypercube into 2 clusters per class. There are 5 informative attributes, 5 redundant attributes, and 30 useless attributes. There is interdependence between these attributes, and a randomly selected 5\% of the labels are flipped to increase the difficulty of the classification task.

    \item Higgs~\cite{baldi2014searching,Dua:2019} contains 11,000,000 signal processes (5,829,123 Higgs bosons) characterized by 22 kinematic properties measured by detectors in a particle accelerator and 7 attributes derived from those properties. The goal is to distinguish between a background signal process and a Higgs bosons process.
\end{itemize}

\begin{table*}[h]
\caption{Dataset summary including the main predictive performance metric used for each dataset, either average precision (AP) for datasets whose positive label percentage $<1$\%, AUC for datasets between [1\%, 20\%], or accuracy (Acc.) for all remaining datasets.}
\vskip 0.15in
\label{tab:datasets}
\centering
\begin{tabular}{lrrrrrrrr}
\toprule
\textbf{Dataset} & \textbf{No. train} & \textbf{Pos. label \%} &
\textbf{No.\ test} & \textbf{Pos.\ label \%} & \textbf{No.\ attributes} & \textbf{Metric} \\
\midrule
Surgical        & 11,708    & 25.30 & 2,927     & 25.00 & 90    & Acc. \\
Vaccine         & 21,365    & 46.60  & 5,342    & 45.60 & 185   & Acc. \\
Adult           & 32,561    & 24.00 & 16,281    & 23.60 & 107   & Acc. \\
Bank Marketing  & 32,951    & 11.40 & 8,237     & 10.90 & 63    & AUC  \\
Flight Delays   & 80,000    & 18.90 & 20,000    & 19.50 & 648   & AUC  \\
Diabetes        & 81,412    & 46.00 & 20,353    & 46.50 & 253   & Acc. \\
No Show         & 88,422    & 20.14 & 22,105    & 20.41 & 99    & AUC  \\
Olympics	    & 164,932	& 14.60 & 41,233	& 14.60 & 1,004 & AUC  \\
Census          & 199,523   & 6.20  & 99,762    & 6.20  & 408   & AUC  \\
Credit Card	    & 227,846	& 0.18  & 56,961    & 0.17  & 29    & AP   \\
CTR             & 800,000   & 2.89  & 200,000   & 2.98  & 13    & AUC  \\
Twitter         & 800,000   & 16.96 & 200,000   & 16.83 & 15    & AUC  \\
Synthetic       & 800,000   & 50.00 & 200,000   & 50.00 & 40    & Acc. \\
Higgs           & 8,800,000 & 53.00 & 2,200,000 & 53.00 & 28    & Acc. \\
\bottomrule
\end{tabular}
\end{table*}

For each dataset, we generate one-hot encodings for any categorical variable and leave all numeric and binary variables as is. For any dataset without a designated train and test split, we randomly sample 80\% of the data for training and use the rest for testing. Table~\ref{tab:datasets} summarizes the datasets after preprocessing.

\newpage

\subsection{Predictive Performance of DaRE Forests}\label{appendix_subsec:predictive_performance}

If extremely randomized trees exhibit the same predictive performance as their greedy counterparts, then adding and removing data can be done by simply updating class counts at the leaves and only retraining if a chosen threshold is no longer within the range of a chosen split attribute for a given decision node. Thus, this section compares the predictive performance of a G-DaRE forest against:
\begin{itemize}
    \item Random Trees: Extremely randomized trees~\cite{geurts2006extremely} in which each decision node selects an attribute to split on uniformly at random, and then selects the threshold by sampling a value in that attribute's [min, max] range uniformly at random.
    
    \item Extra Trees: Similar to the extremely randomized trees model~\cite{geurts2006extremely}, except each decision node selects $\lfloor \sqrt{p} \rfloor$ attributes at random; a threshold is then selected for each attribute by sampling a value in that attribute's [min, max] range uniformly at random. Then, a split criterion such as Gini index or mutual information is computed for each attribute-threshold pair, and the best threshold is chosen as the split for that node.
    
    \item SKLearn RF: Standard RF implementation from Scikit-Learn~\cite{scikit-learn}.
    
    \item SKLearn RF (w/ bootstrap): Standard RF implementation from Scikit-Learn~\cite{scikit-learn} with bootstrapping.
\end{itemize}

Table~\ref{tab:predictive_performance} reports the predictive performance of each model on the test set after tuning using 5-fold cross-validation. We tune the number of trees in the forest using values [10, 25, 50, 100, 250], and the maximum depth using values [1, 3, 5, 10, 20]. The maximum number of randomly selected attributes to consider at each split is set to $\lfloor \sqrt{p} \rfloor$. For the G-DaRE model, we also tune the number of thresholds to consider for each attribute, $k$, using values [5, 10, 25, 50]. We use 50\%, 25\%, 2.5\%, and 2.5\% of the training data to tune the Twitter, Synthetic, Click-Through Rate, and Higgs datasets, respectively, and 100\% for all other datasets. Selected values for all hyperparameters are in Table~\ref{tab:dart_hyperparameters}.

We find the predictive performance of the Random Trees and Extra Trees models to be consistently worse than the SKLearn and G-DaRE models. We also find that bootstrapping has a negligible effect on the SKLearn models. Finally, we observe that the predictive performance of the G-DaRE model is nearly identical to that of SKLearn RF, in which their scores are within 0.2\% on 9/14 datasets, 0.4\% on 1/14 datasets, and G-DaRE RF is significantly better than SKLearn RF on the Surgical, Flight Delays, Olympics, and Credit Card datasets.

\begin{table}[h]
\centering
\caption{Predictive performance comparison of G-DaRE RF to: an extremely randomized trees model~(Random Trees)~\cite{geurts2006extremely}, an Extra Trees~\cite{geurts2006extremely} model, and a popular and widely used random forest implementation from Scikit-Learn~(SKLearn) with and without bootstrapping. The numbers in each cell represent either average precision, AUC, or accuracy as specified by Table~\ref{tab:datasets}; results are averaged over five runs and the standard error is shown.
}
\vskip 0.15in
\label{tab:predictive_performance}
\begin{tabular}{lrrrrr}
\toprule
\textbf{Dataset} &
  \textbf{Random Trees} &
  \textbf{Extra Trees} &
  \multicolumn{1}{r}{\textbf{SKLearn RF}} &
  \multicolumn{1}{r}{\textbf{\begin{tabular}[c]{r}SKLearn RF\\ (w/ bootstrap)\end{tabular}}} &
  \textbf{G-DaRE RF} \\ \midrule
Surgical       & 0.783 $\pm$ 0.001 & 0.805 $\pm$ 0.001 & 0.848 $\pm$ 0.001 & 0.846 $\pm$ 0.001 & \textbf{0.867 $\pm$ 0.001} \\
Vaccine        & 0.769 $\pm$ 0.001 & 0.795 $\pm$ 0.001 & \textbf{0.796 $\pm$ 0.001} & 0.793 $\pm$ 0.002 & 0.794 $\pm$ 0.001 \\
Adult          & 0.802 $\pm$ 0.003 & 0.847 $\pm$ 0.001 & \textbf{0.863 $\pm$ 0.000} & \textbf{0.863 $\pm$ 0.000} & 0.862 $\pm$ 0.001 \\
Bank Marketing & 0.879 $\pm$ 0.001 & 0.924 $\pm$ 0.000 & \textbf{0.940 $\pm$ 0.001} & \textbf{0.940 $\pm$ 0.001} & \textbf{0.940 $\pm$ 0.001} \\
Flight Delays  & 0.650 $\pm$ 0.009 & 0.725 $\pm$ 0.001 & 0.729 $\pm$ 0.001 & 0.729 $\pm$ 0.000 & \textbf{0.739 $\pm$ 0.000} \\
Diabetes       & 0.551 $\pm$ 0.003 & 0.631 $\pm$ 0.001 & 0.643 $\pm$ 0.000 & 0.642 $\pm$ 0.001 & \textbf{0.645 $\pm$ 0.000} \\
No Show        & 0.694 $\pm$ 0.001 & 0.710 $\pm$ 0.000 & 0.732 $\pm$ 0.000 & 0.731 $\pm$ 0.000 & \textbf{0.736 $\pm$ 0.000} \\
Olympics       & 0.835 $\pm$ 0.001 & 0.820 $\pm$ 0.001 & 0.819 $\pm$ 0.001 & 0.820 $\pm$ 0.000 & \textbf{0.871 $\pm$ 0.000} \\
Credit Card    & 0.799 $\pm$ 0.002 & 0.840 $\pm$ 0.004 & 0.837 $\pm$ 0.002 & 0.831 $\pm$ 0.005 & \textbf{0.846 $\pm$ 0.001} \\
Census         & 0.915 $\pm$ 0.001 & 0.936 $\pm$ 0.000 & 0.945 $\pm$ 0.000 & 0.945 $\pm$ 0.000 & \textbf{0.946 $\pm$ 0.000} \\
CTR            & 0.668 $\pm$ 0.001 & 0.683 $\pm$ 0.000 & \textbf{0.702 $\pm$ 0.000} & 0.700 $\pm$ 0.000 & 0.701 $\pm$ 0.000 \\
Twitter        & 0.883 $\pm$ 0.001 & 0.923 $\pm$ 0.001 & \textbf{0.943 $\pm$ 0.000} & 0.942 $\pm$ 0.000 & \textbf{0.943 $\pm$ 0.000} \\
Synthetic      & 0.793 $\pm$ 0.002 & 0.909 $\pm$ 0.001 & \textbf{0.946 $\pm$ 0.001} & 0.945 $\pm$ 0.000 & 0.945 $\pm$ 0.000 \\
Higgs          & 0.608 $\pm$ 0.001 & 0.700 $\pm$ 0.000 & \textbf{0.746 $\pm$ 0.000} & 0.744 $\pm$ 0.000 & 0.744 $\pm$ 0.000 \\
\bottomrule
\end{tabular}
\end{table}

\begin{table}[h]
\centering
\caption{Hyperparameters selected for the G-DaRE and R-DaRE~(using error tolerances of 0.1\%, 0.25\%, 0.5\%, and 1.0\%) models. The number of trees~($T$), maximum depth~($d_{\max}$), and the number of thresholds considered per attribute~($k$) are found using 5-fold cross-validation using a greedily-built model~(i.e. G-DaRE RF). To build the R-DaRE model, the values for $T$, $d_{\max}$, and $k$ found in the previous step are held fixed, and the value for \dr is found by incrementing its value by one starting from zero until its 5-fold cross-validation score exceeds the specified error tolerance as compared to the cross-validation score of the G-DaRE model.}
\vskip 0.15in
\label{tab:dart_hyperparameters}
\begin{tabular}{@{}lrrrrrrr@{}}
\toprule
& \multicolumn{3}{c}{\textbf{G-DaRE} \& \textbf{R-DaRE}}
& \multicolumn{4}{c}{\textbf{R-DaRE Only}} \\
\cmidrule(lr){2-4}\cmidrule(lr){5-8}
\textbf{Dataset} &
  $T$ & $d_{\max}$ & $k$ &
  \begin{tabular}[c]{@{}r@{}}\dr\\ \textbf{(0.1\%)}\end{tabular}  &
  \begin{tabular}[c]{@{}r@{}}\dr\\ \textbf{(0.25\%)}\end{tabular} &
  \begin{tabular}[c]{@{}r@{}}\dr\\ \textbf{(0.5\%)}\end{tabular}  &
  \begin{tabular}[c]{@{}r@{}}\dr\\ \textbf{(1.0\%)}\end{tabular}  \\ \midrule
Surgical           & 100 & 20 & 25       & 0  & 1  & 2  & 4  \\
Vaccine            & 50  & 20 & 5        & 5  & 7  & 11 & 14 \\
Adult              & 50  & 20 & 5        & 10 & 13 & 14 & 16 \\
Bank Marketing     & 100 & 20 & 25       & 6  & 9  & 12 & 14 \\
Flight Delays      & 250 & 20 & 25       & 1  & 3  & 5  & 10 \\
Diabetes           & 250 & 20 & 5        & 7 & 10  & 12 & 15 \\
No Show            & 250 & 20 & 10       & 1  & 3  & 6  & 10 \\
Olympics           & 250 & 20 & 5        & 0  & 1  & 2  & 3  \\
Census             & 100 & 20 & 25       & 6  & 9  & 12 & 16 \\
Credit Card        & 250 & 20 & 5        & 5  & 8  & 14 & 17 \\
CTR                & 100 & 10 & 50       & 2  & 3  & 4  & 6  \\
Twitter            & 100 & 20 & 5        & 2  & 4  & 7  & 11 \\
Synthetic          & 50  & 20 & 10       & 0  & 2  & 3  & 5  \\
Higgs              & 50 & 20  & 10       & 1  & 3  & 6  & 9  \\
\bottomrule
\end{tabular}
\end{table}

\begin{table}[h!]
\centering
\caption{Training times~(in seconds) for the G-DaRE model using the hyperparameters selected in Table~\ref{tab:dart_hyperparameters}. Mean and standard deviations~(S.D.) are computed over five runs.}
\vskip 0.15in
\label{tab:train_times}
\begin{tabular}{lrr}
\toprule
\textbf{Dataset} & \textbf{Mean} & \textbf{S.D.} \\
\midrule
Surgical           & 5.68     & 2.97   \\
Vaccine            & 17.08    & 11.86  \\
Adult              & 6.76     & 1.17   \\
Bank Marketing     & 8.79     & 3.37   \\
Flight Delays      & 262.00   & 50.39  \\
Diabetes           & 141.91   & 39.12  \\
No Show            & 77.65    & 20.33  \\
Olympics           & 596.27   & 157.70 \\
Census             & 127.40   & 9.57   \\
Credit Card        & 616.65   & 166.00 \\
Twitter            & 152.34   & 12.32  \\
Synthetic          & 732.05   & 231.70 \\
CTR                & 121.64   & 37.13  \\
Higgs              & 5,016.44 & 146.34 \\
\bottomrule
\end{tabular}
\end{table}

\newpage
\phantom{dummy}
\newpage

\newpage

\subsection{Effect of \dr on Deletion Efficiency}
\label{appendix_subsec:d_rmax}

\newcommand{\w}{0.95}

Figure~\ref{fig:appendix_d_rmax} presents additional results on the effect \dr has on deletion efficiency for different datasets.

\begin{figure}[h!]
    \centering
    \subfloat[Surgical]{\includegraphics[width=\w\textwidth]{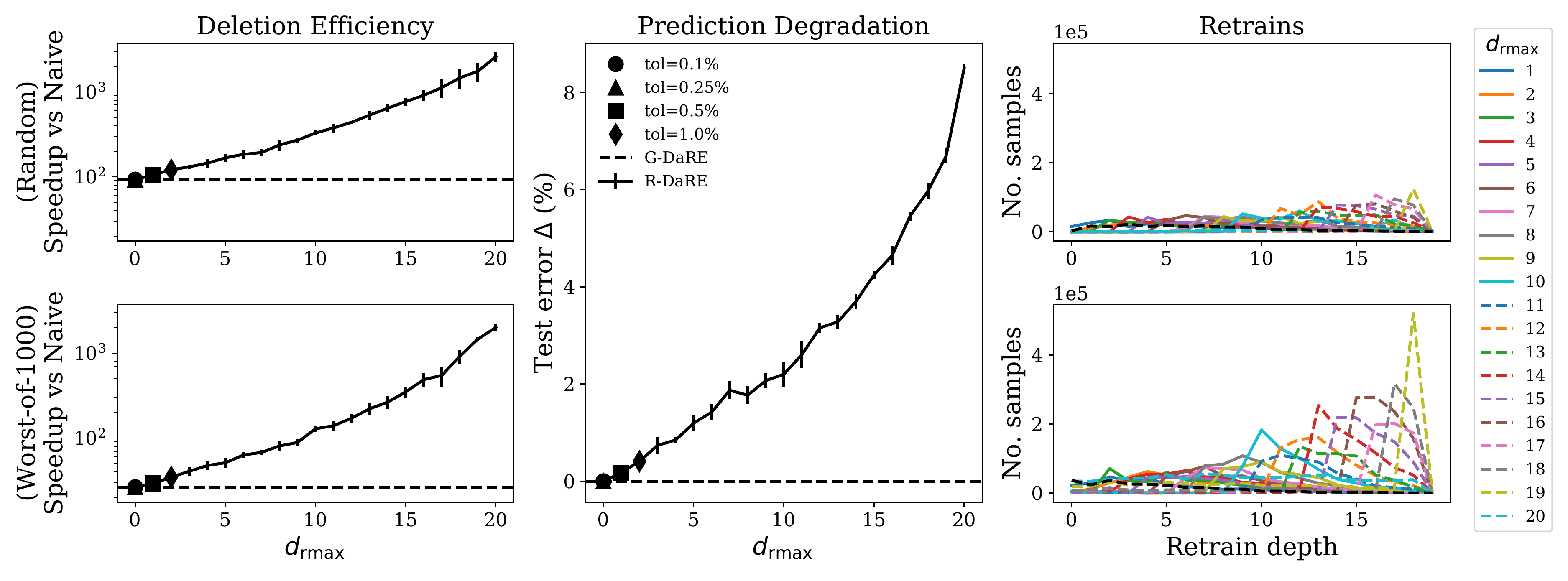}}\hfill
    \subfloat[Vaccine]{\includegraphics[width=\w\textwidth]{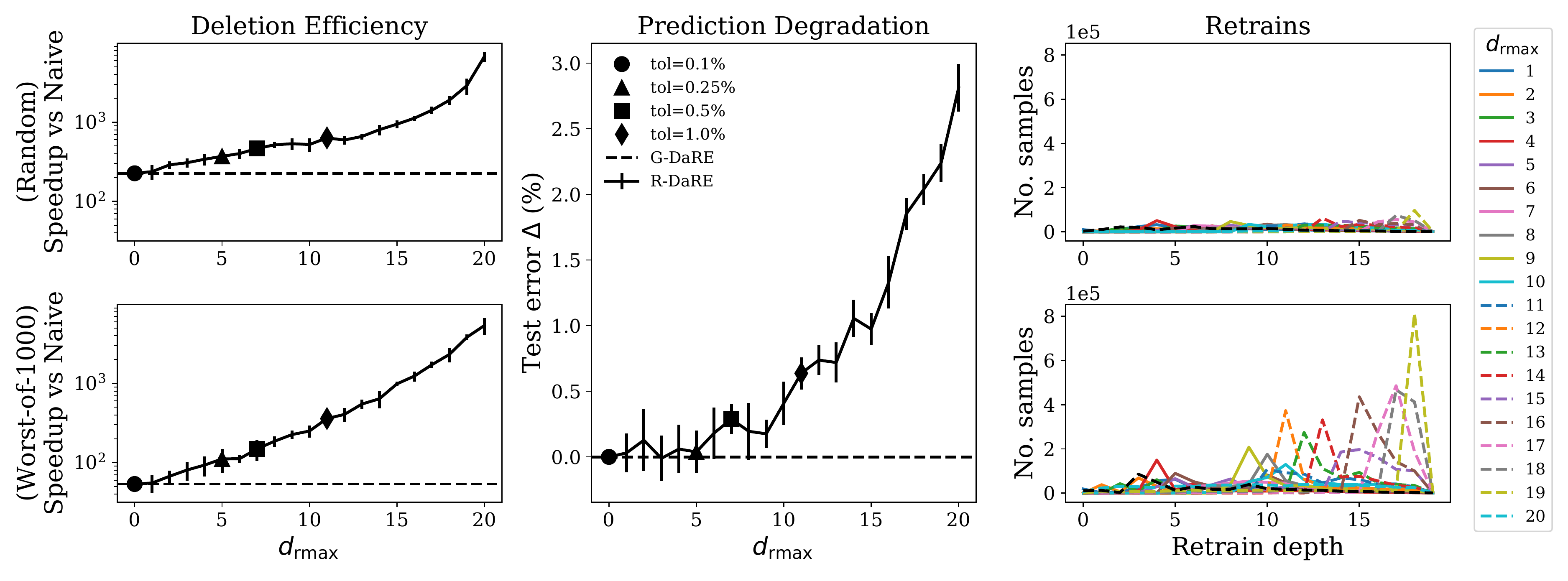}}\hfill
    \subfloat[Adult]{\includegraphics[width=\w\textwidth]{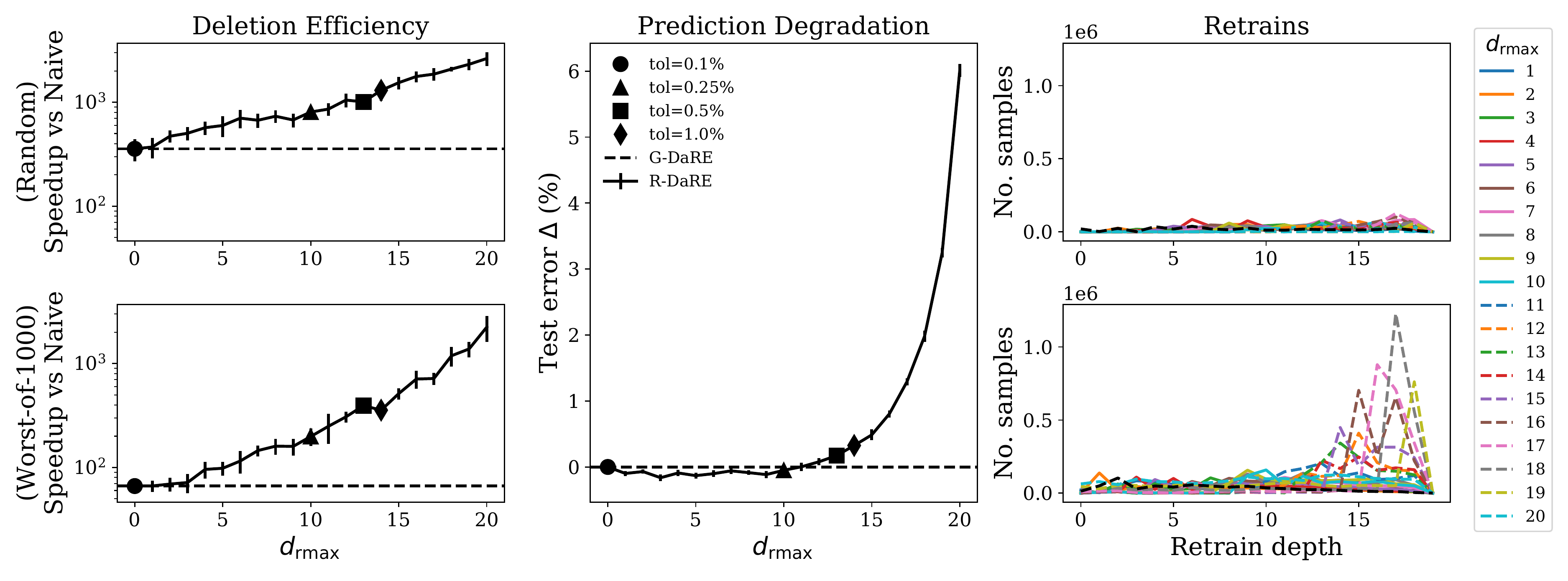}}
    \caption{Effect of \dr on deletion efficiency.}
\end{figure}

\begin{figure}
    \ContinuedFloat
    \centering
    \subfloat[Flight Delays]{\includegraphics[width=\w\textwidth]{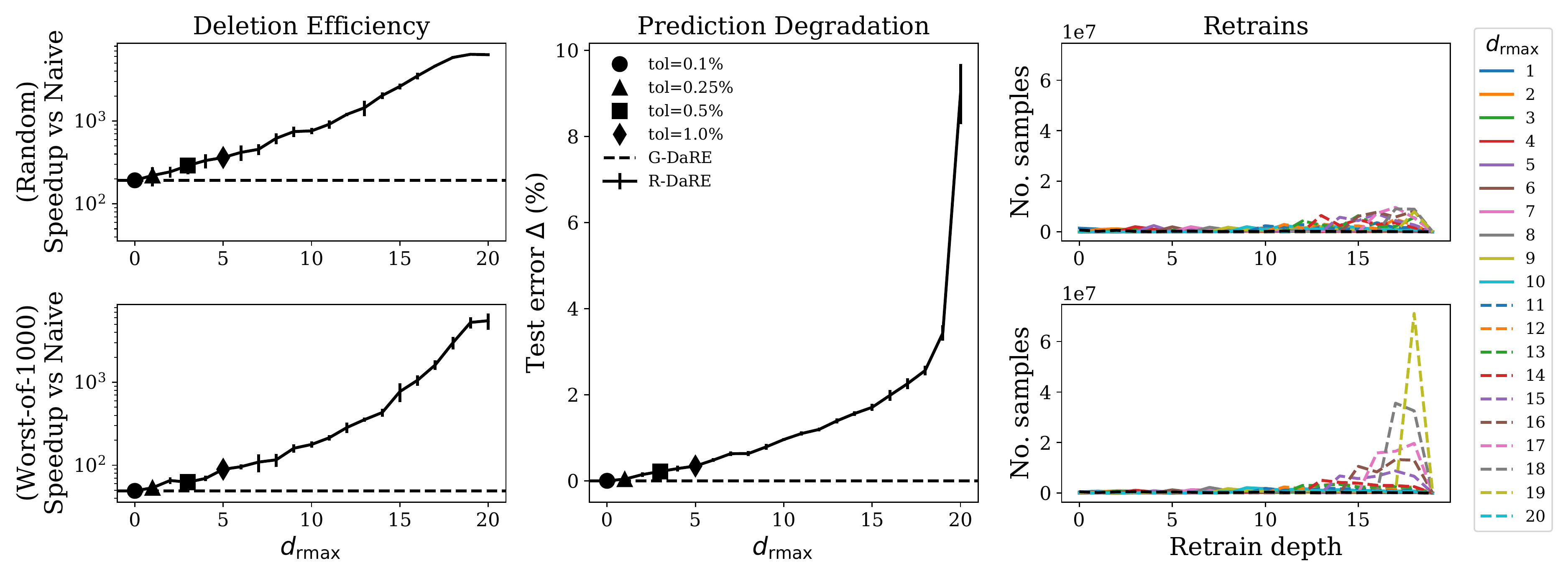}}\hfill
    \subfloat[Diabetes]{\includegraphics[width=\w\textwidth]{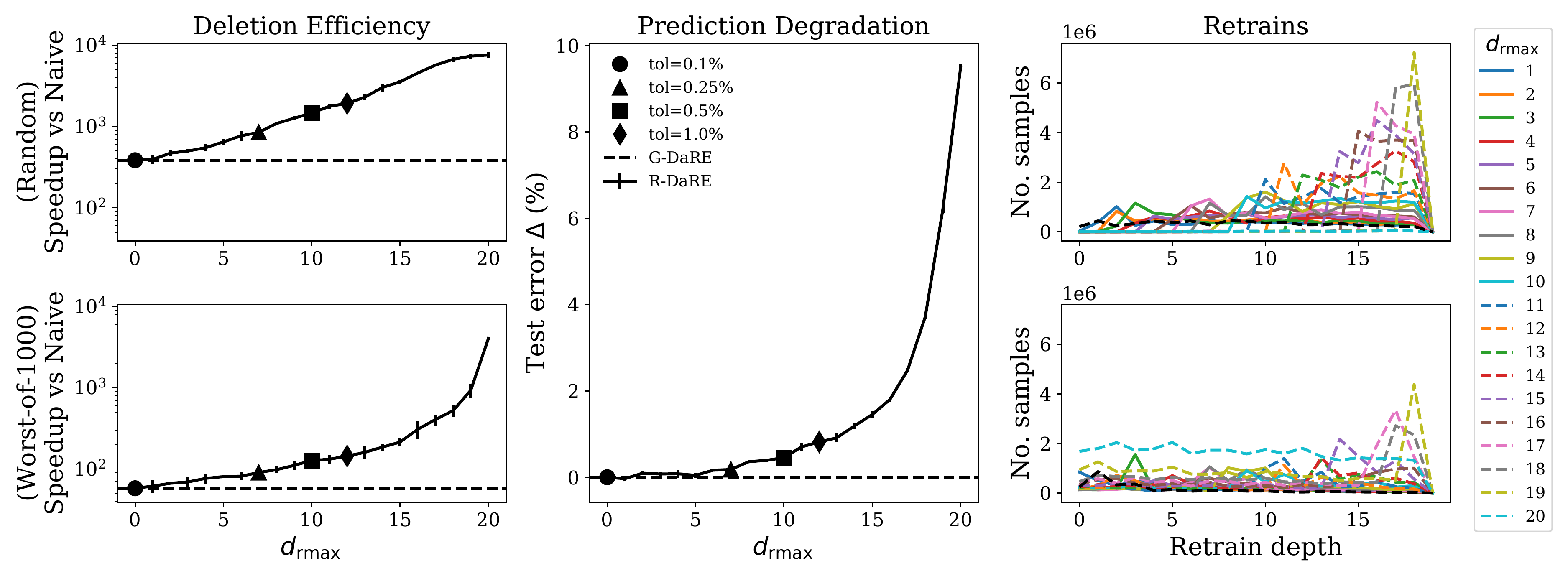}}\hfill
    \subfloat[No Show]{\includegraphics[width=\w\textwidth]{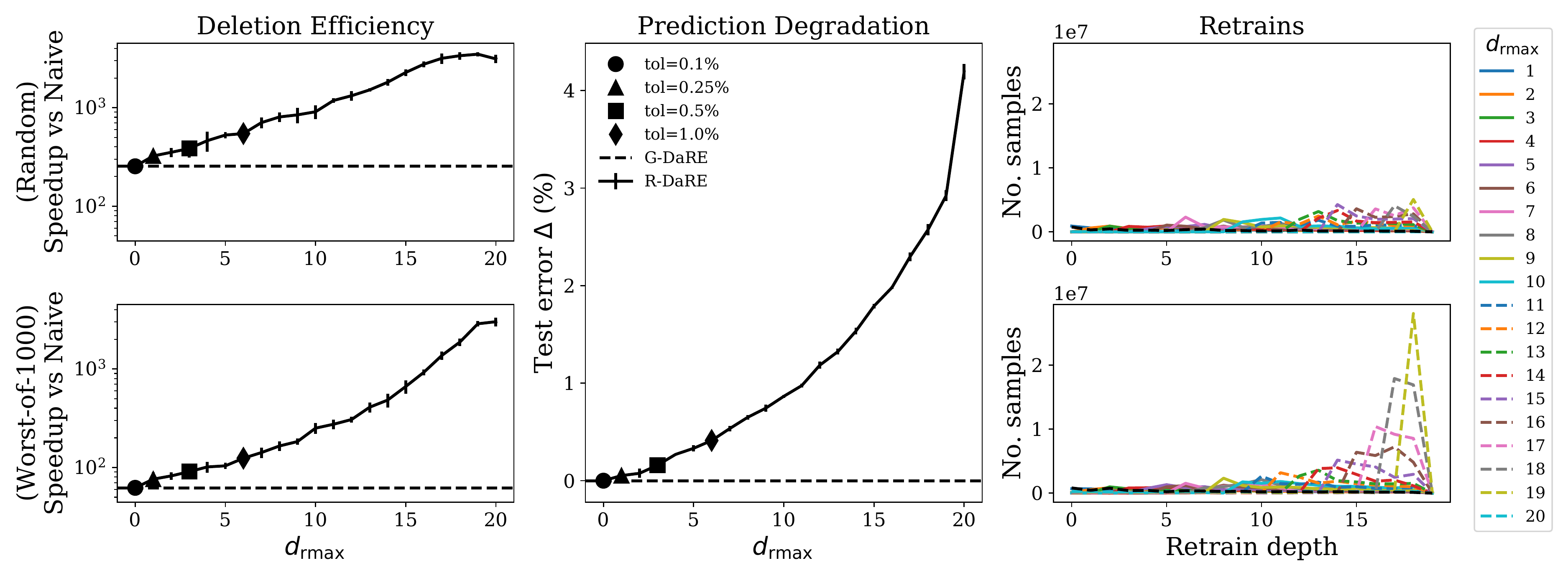}}
    \caption{Effect of \dr on deletion efficiency.}
\end{figure}

\begin{figure}
    \ContinuedFloat
    \centering
    \subfloat[Olympics]{\includegraphics[width=\w\textwidth]{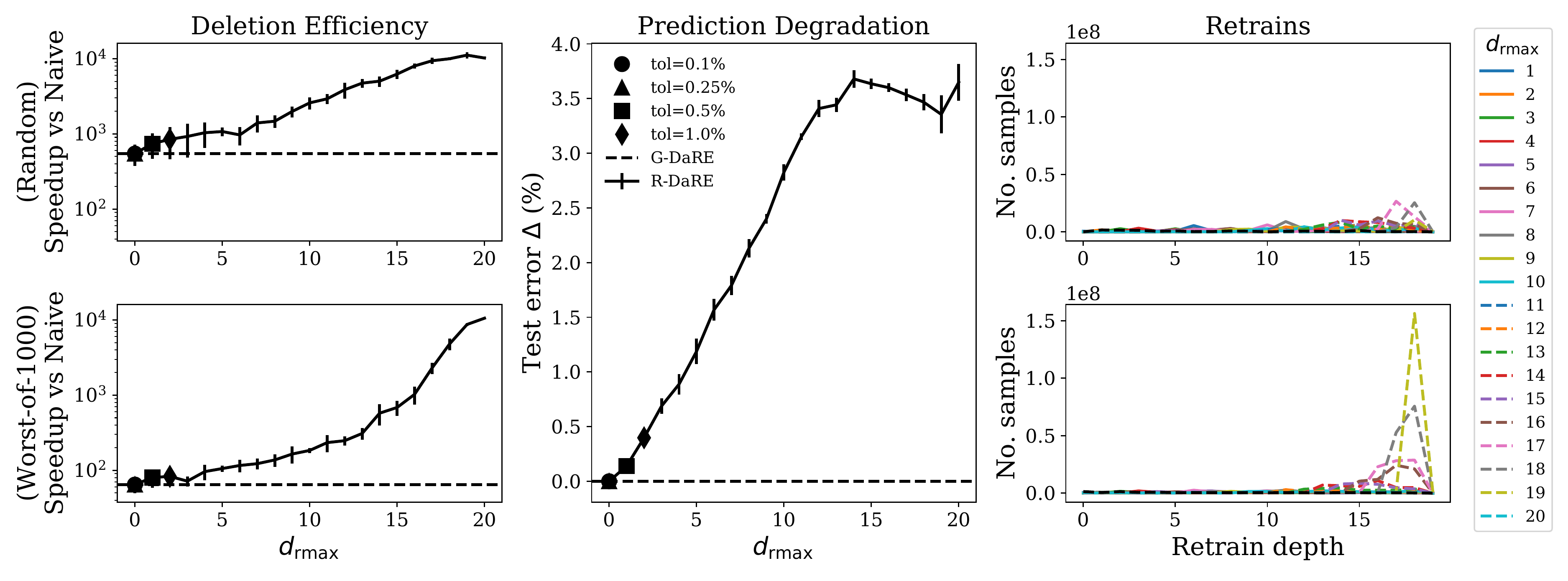}}\hfill
    \subfloat[Census]{\includegraphics[width=\w\textwidth]{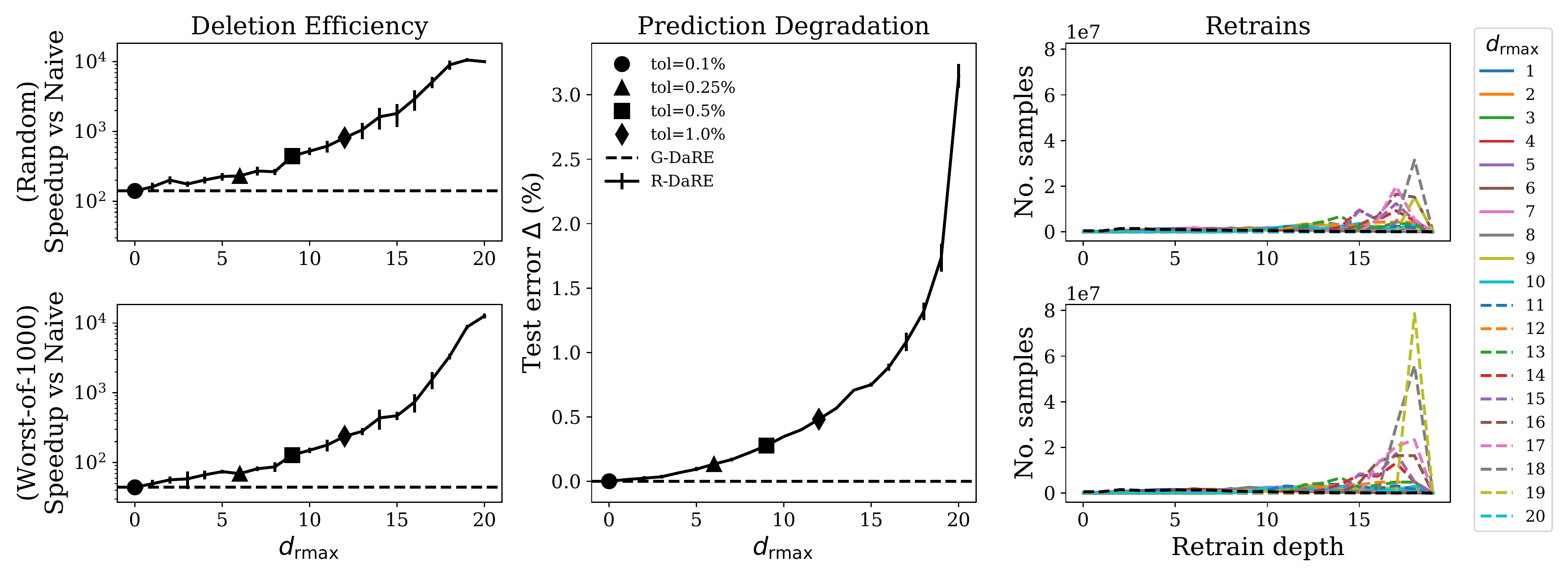}}\hfill
    \subfloat[Credit Card]{\includegraphics[width=\w\textwidth]{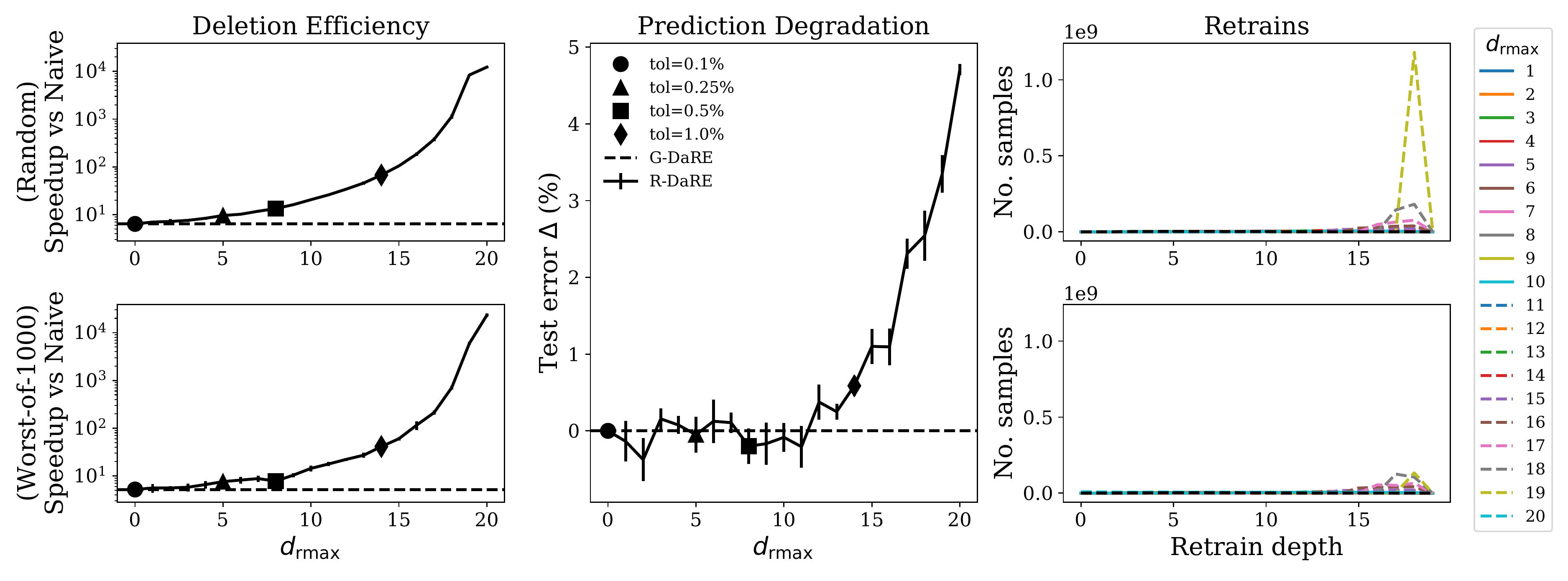}}
    \caption{Effect of \dr on deletion efficiency.}
\end{figure}

\begin{figure}
    \ContinuedFloat
    \centering
    \subfloat[CTR]{\includegraphics[width=\w\textwidth]{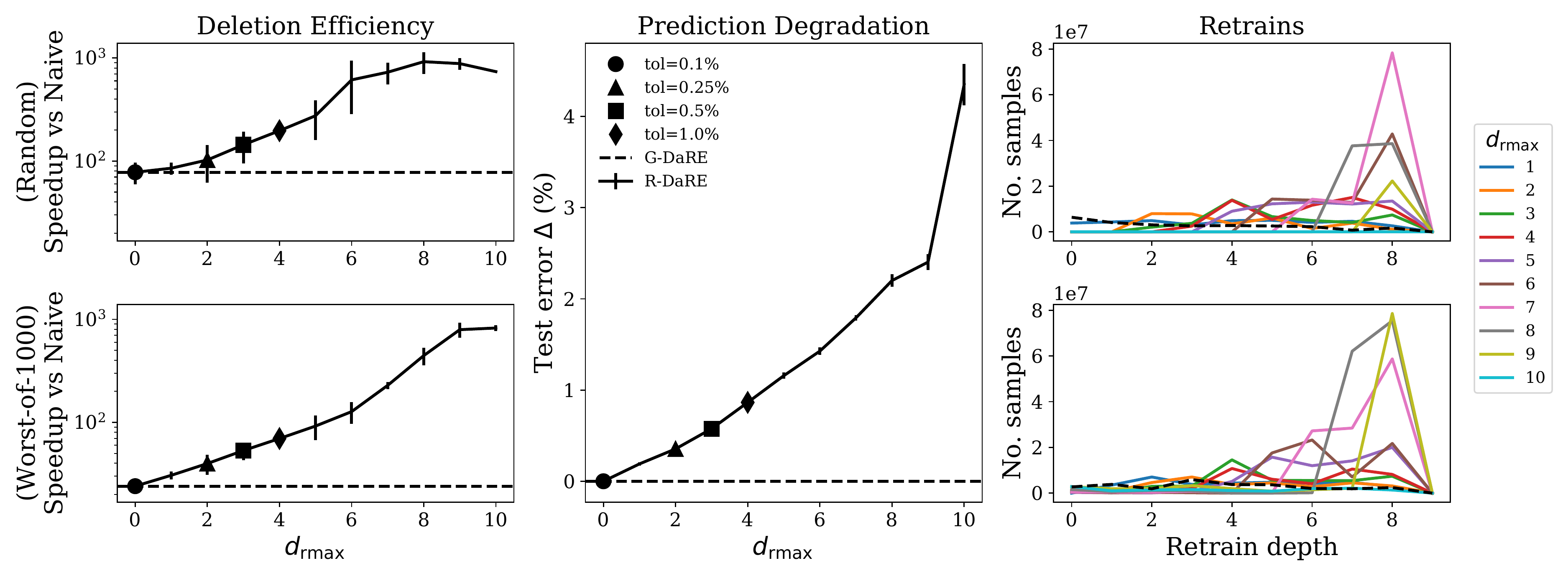}}\hfill
    \subfloat[Twitter]{\includegraphics[width=\w\textwidth]{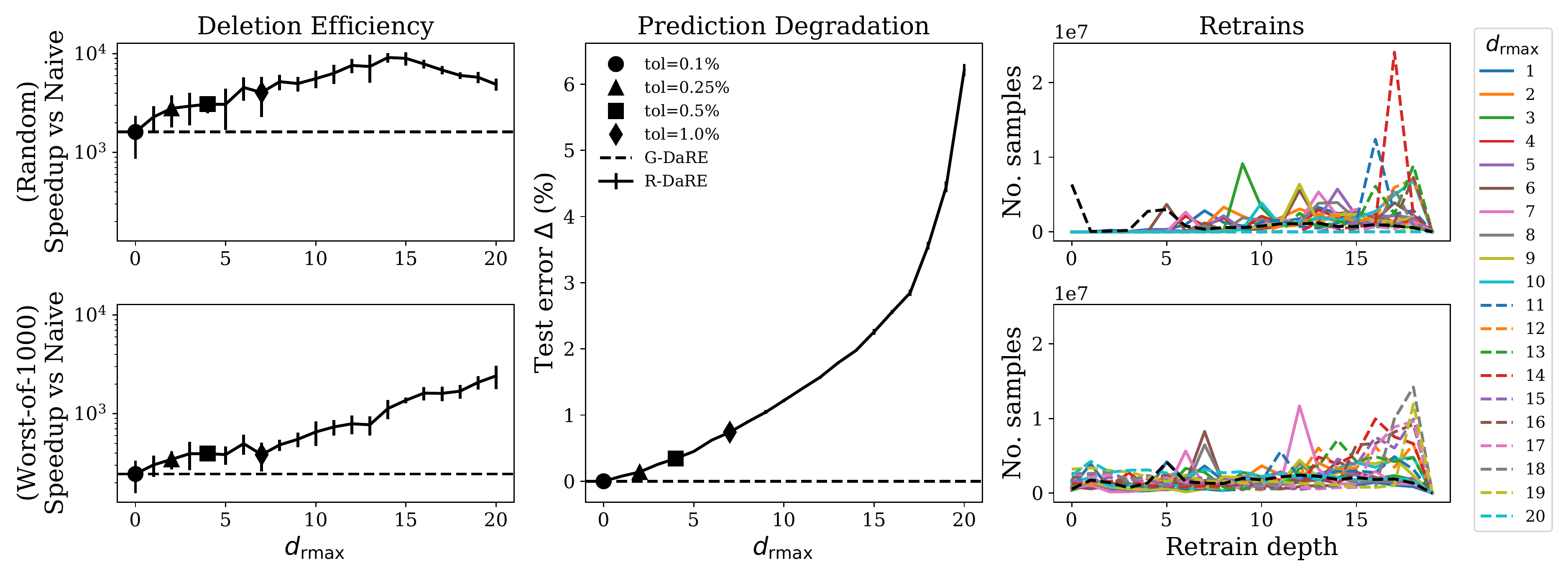}}\hfill
    \subfloat[Synthetic]{\includegraphics[width=\w\textwidth]{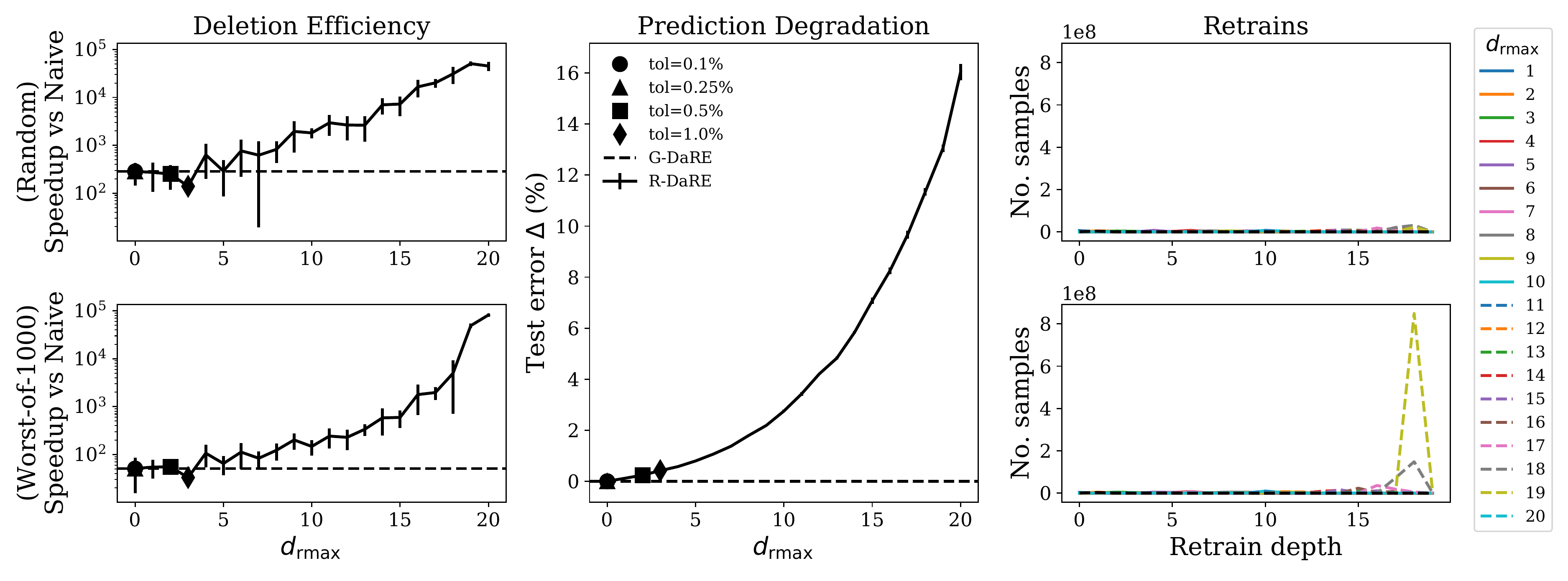}}
    \caption{Effect of \dr on deletion efficiency.}
\end{figure}

\begin{figure}
    \ContinuedFloat
    \centering
    \subfloat[Higgs]{\includegraphics[width=\w\textwidth]{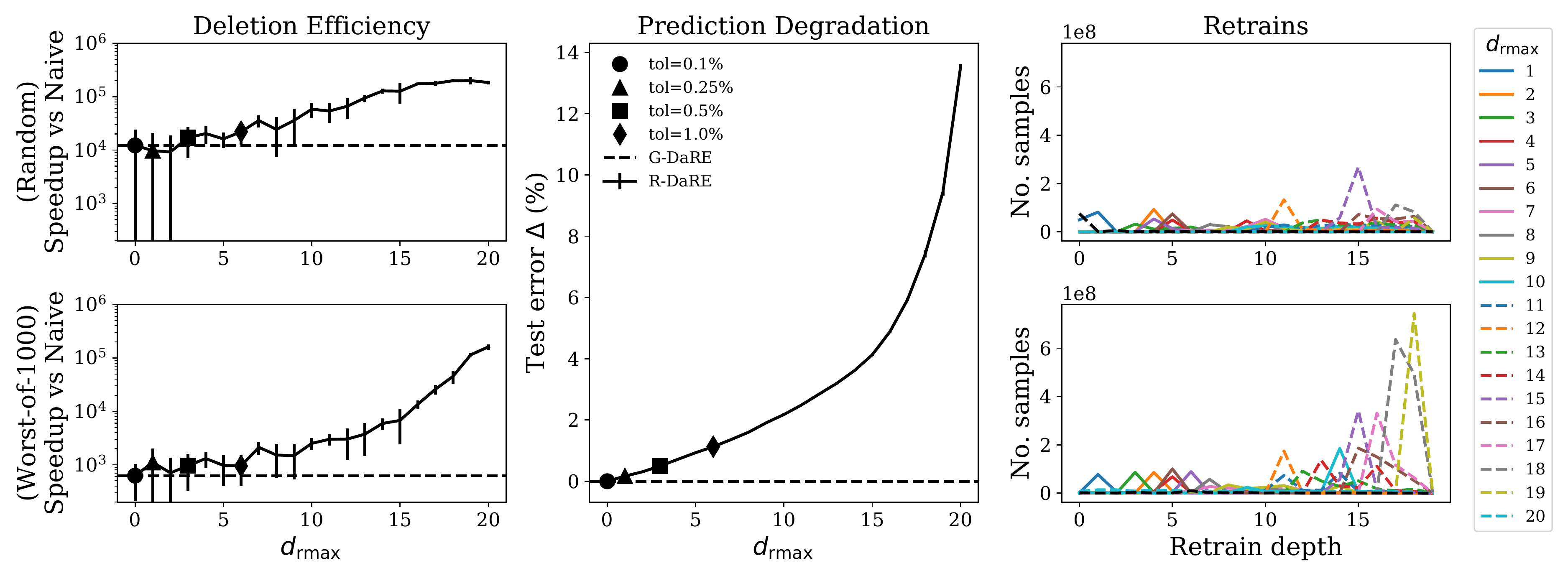}}\hfill
    \caption{Effect of \dr on deletion efficiency.}
    \label{fig:appendix_d_rmax}
\end{figure}

\newpage~\newpage

\subsection{Effect of $k$ on Deletion Efficiency}
\label{appendix_subsec:k}

\newcommand{\ww}{0.5}

Figure~\ref{fig:appendix_k} presents additional results on the effect $k$ has on deletion efficiency for different datasets. For $k$, we tested values [1, 5, 10, 25, 50, 100].

\begin{figure}[h!]
    \centering
    \subfloat[Vaccine: All attributes are binary, thus $k$ has no effect.]{\includegraphics[width=\ww\textwidth]{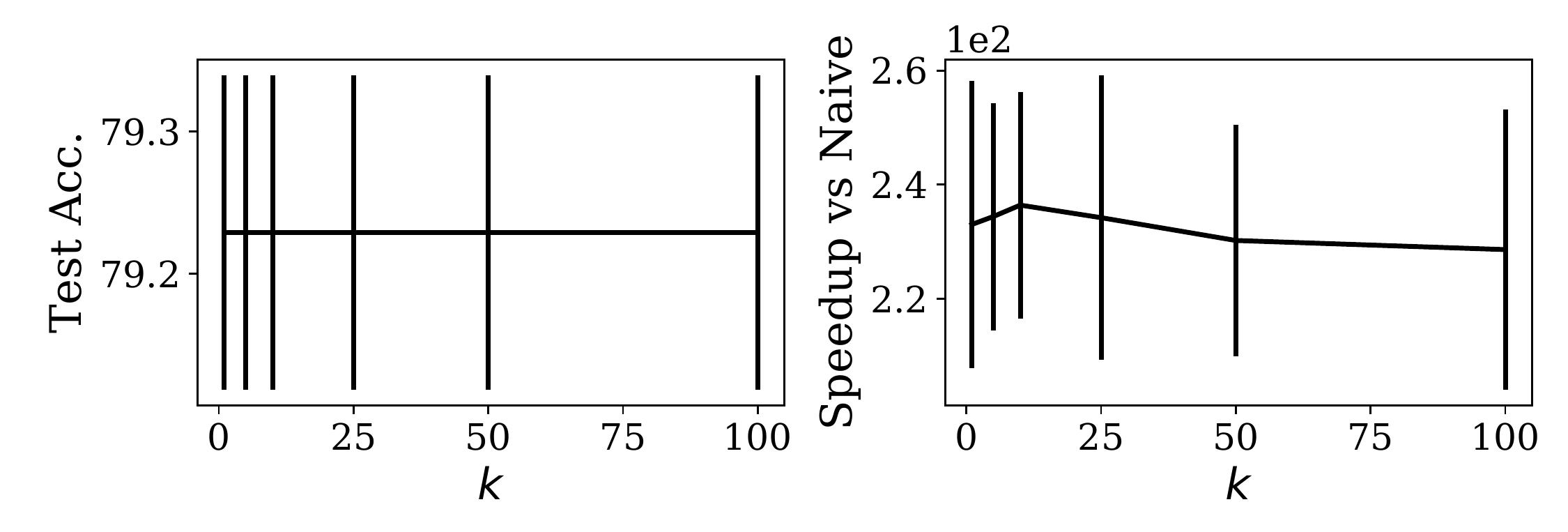}}\hfill
    \subfloat[Adult]{\includegraphics[width=\ww\textwidth]{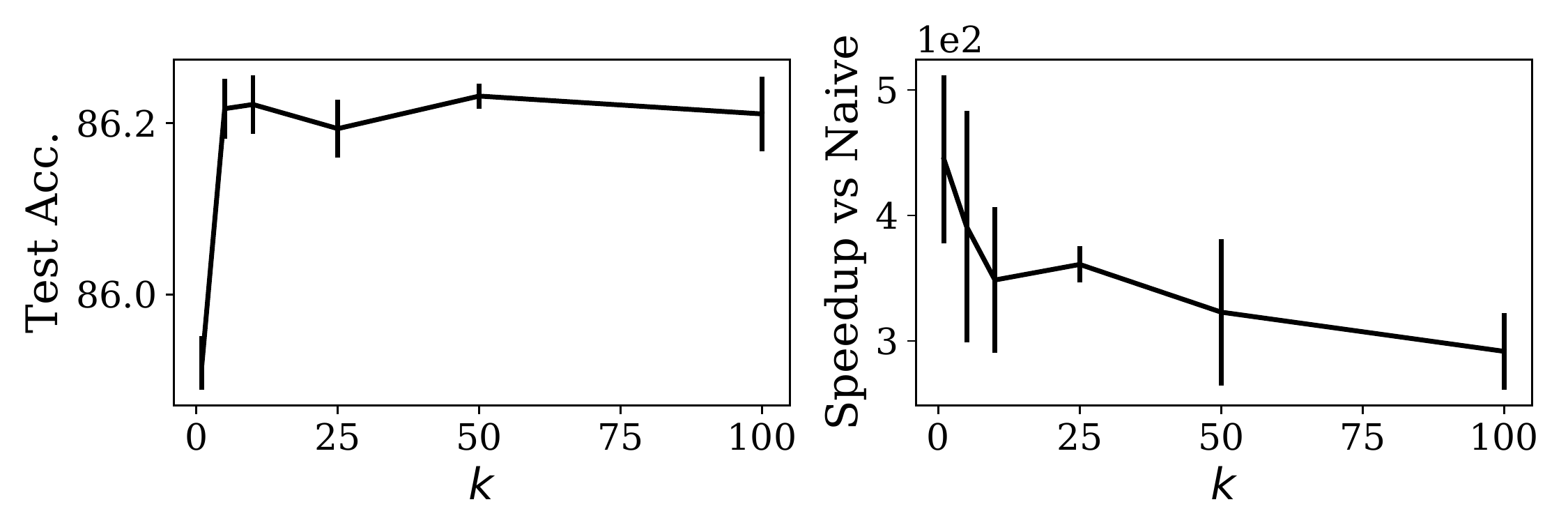}}\hfill
    \subfloat[Bank Marketing]{\includegraphics[width=\ww\textwidth]{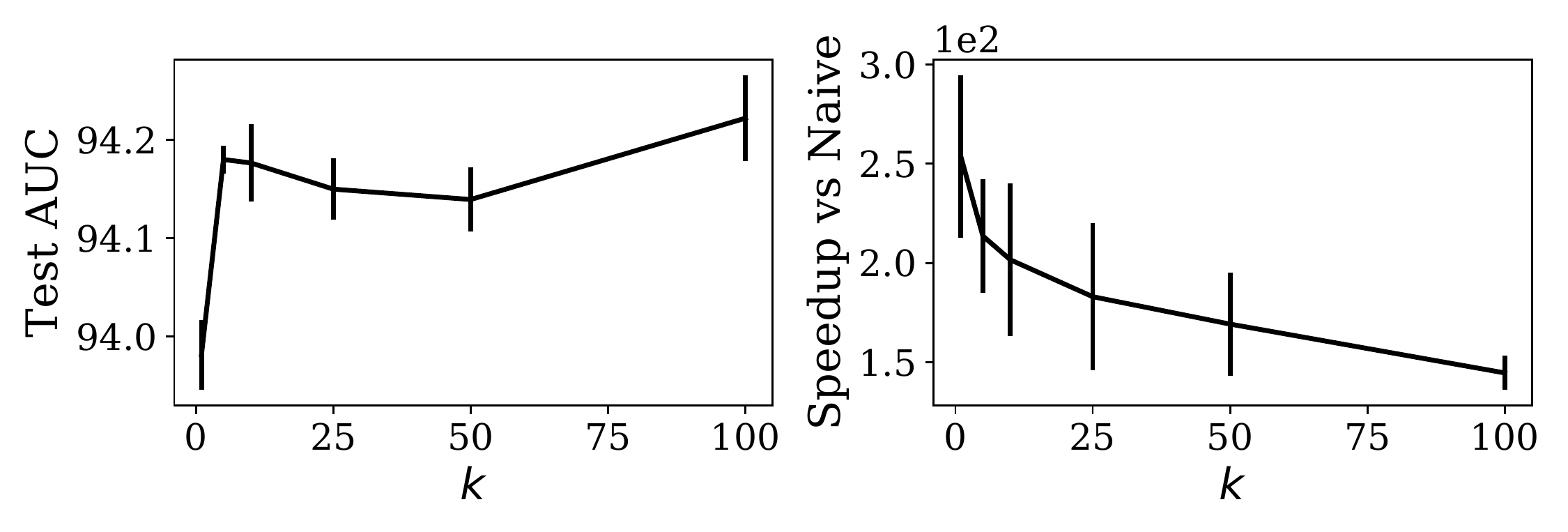}}\hfill
    \subfloat[Flight Delays]{\includegraphics[width=\ww\textwidth]{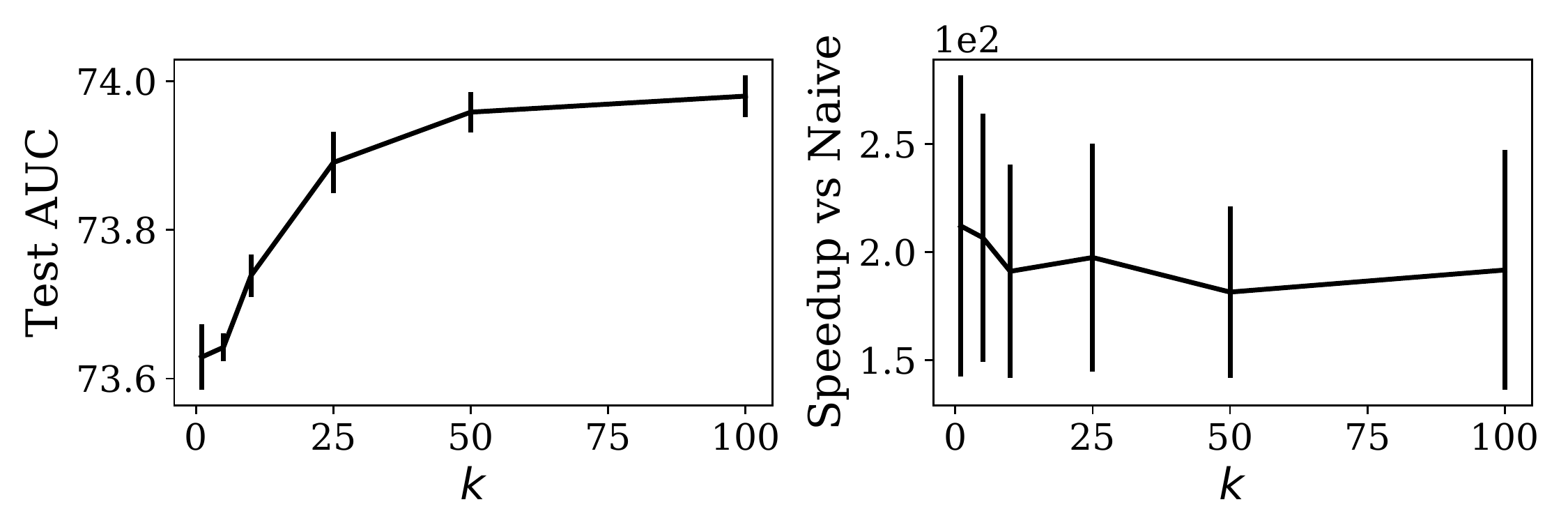}}\hfill
    \subfloat[Diabetes]{\includegraphics[width=\ww\textwidth]{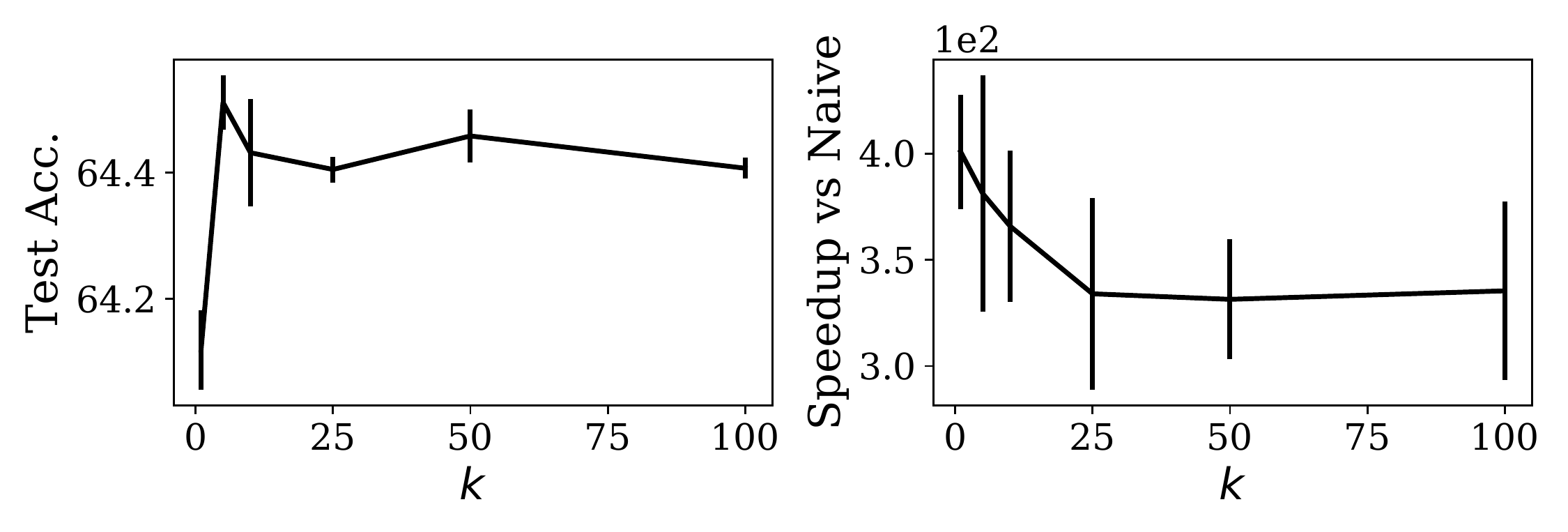}}\hfill
    \subfloat[No Show]{\includegraphics[width=\ww\textwidth]{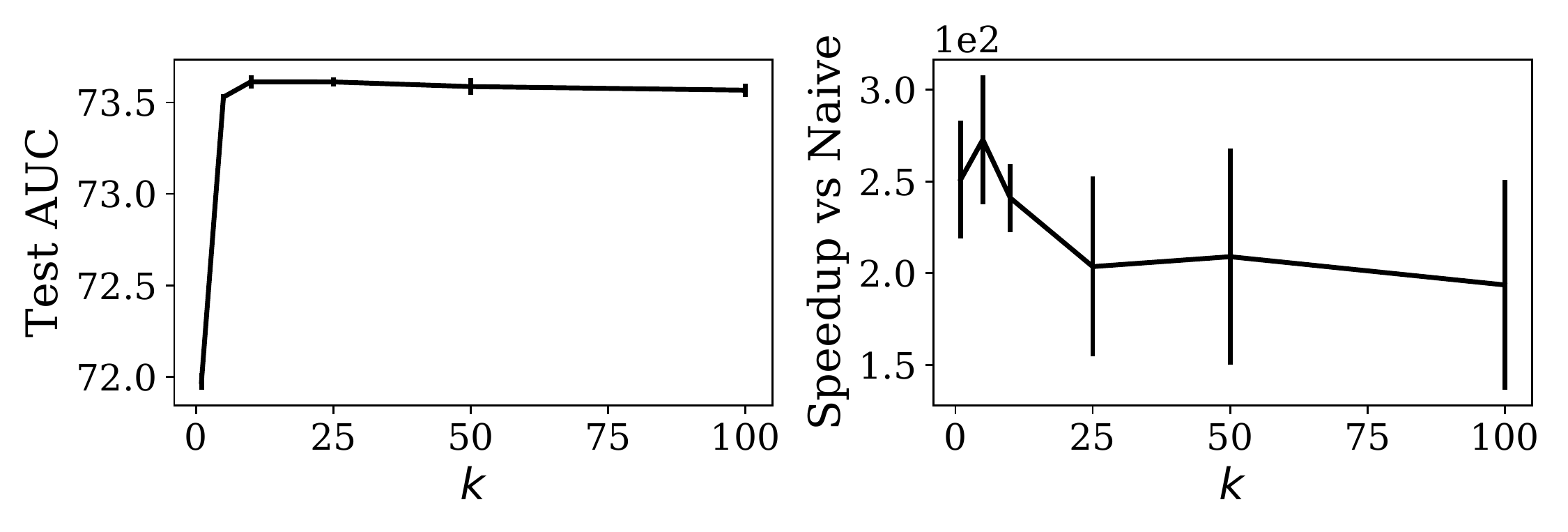}}\hfill
    \subfloat[Olympics: In this case, the randomness induced by a low $k$ value actually helps predictive performance.]{\includegraphics[width=\ww\textwidth]{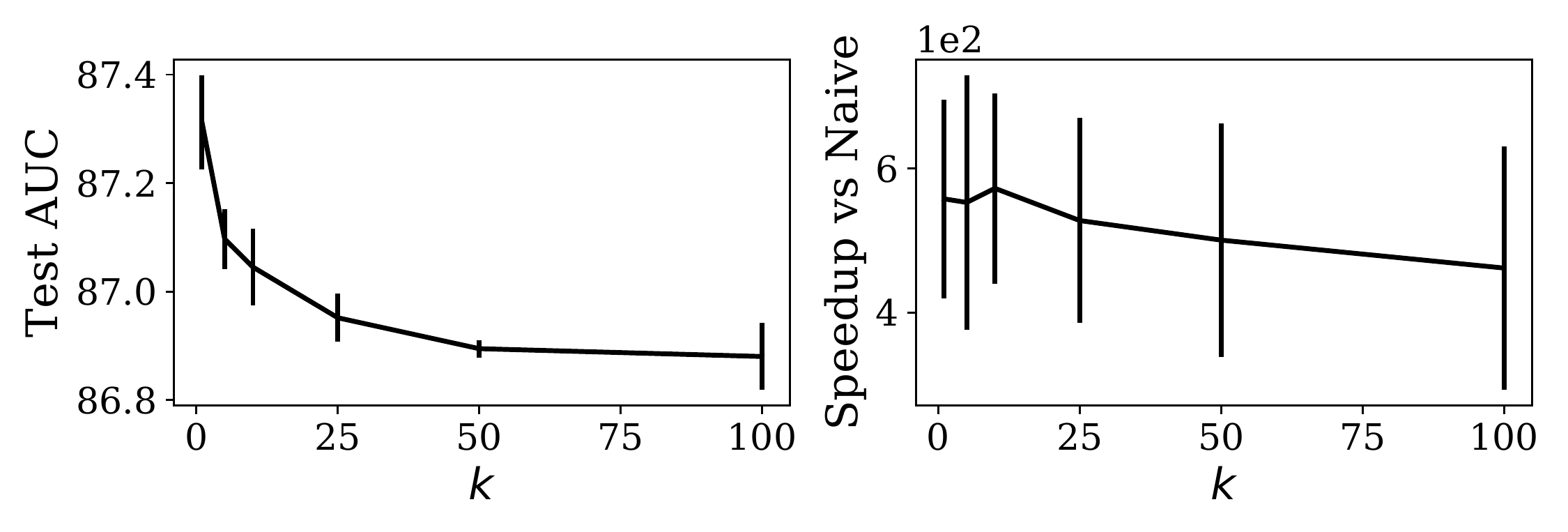}}\hfill
    \subfloat[Census]{\includegraphics[width=\ww\textwidth]{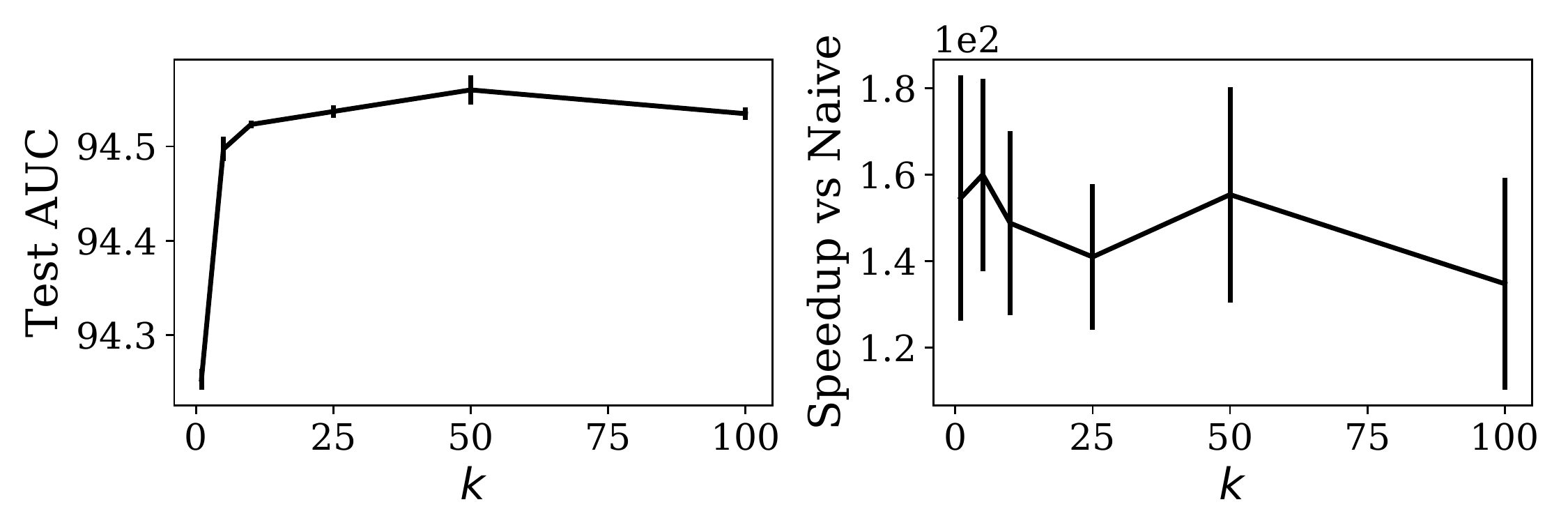}}\hfill
    \subfloat[Credit Card]{\includegraphics[width=\ww\textwidth]{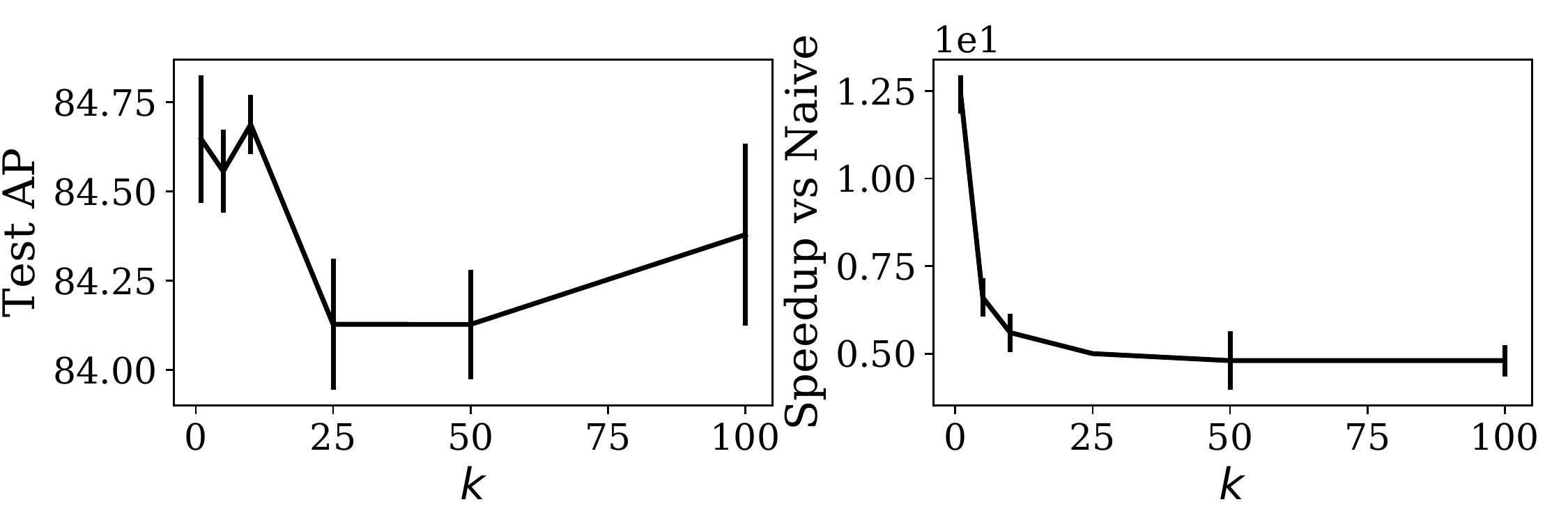}}\hfill
    \subfloat[CTR]{\includegraphics[width=\ww\textwidth]{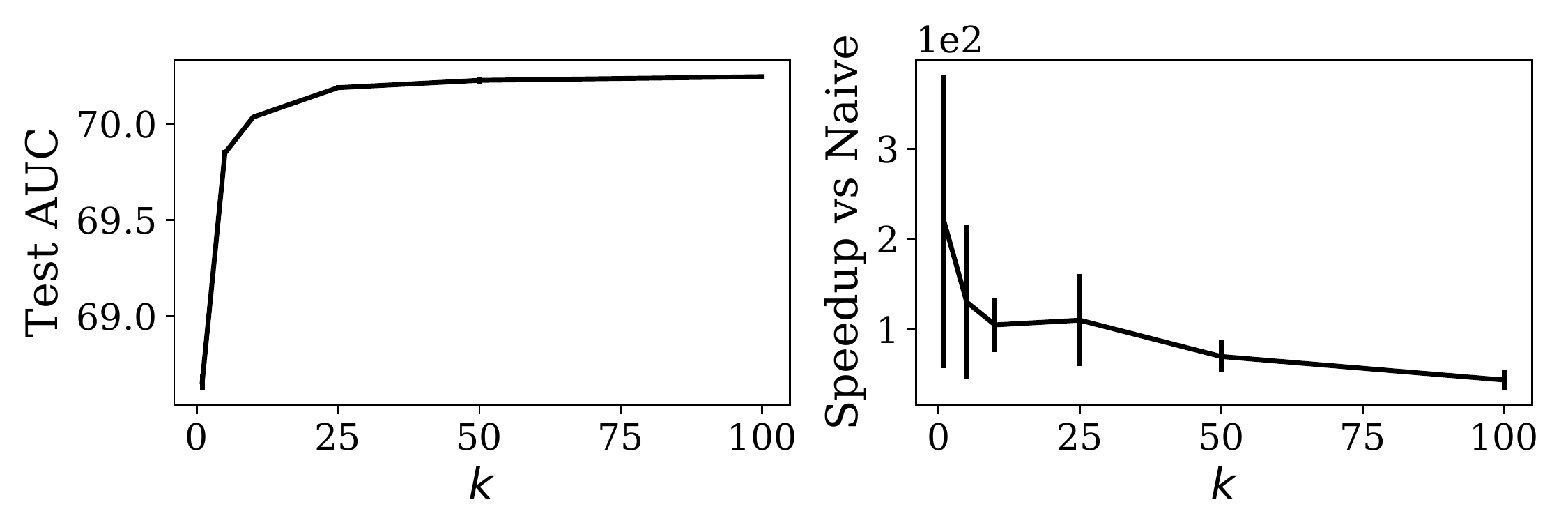}}\hfill
    \caption{Effect of $k$ on deletion efficiency.}
\end{figure}

\begin{figure}[t]
    \ContinuedFloat
    \subfloat[Twitter]{\includegraphics[width=\ww\textwidth]{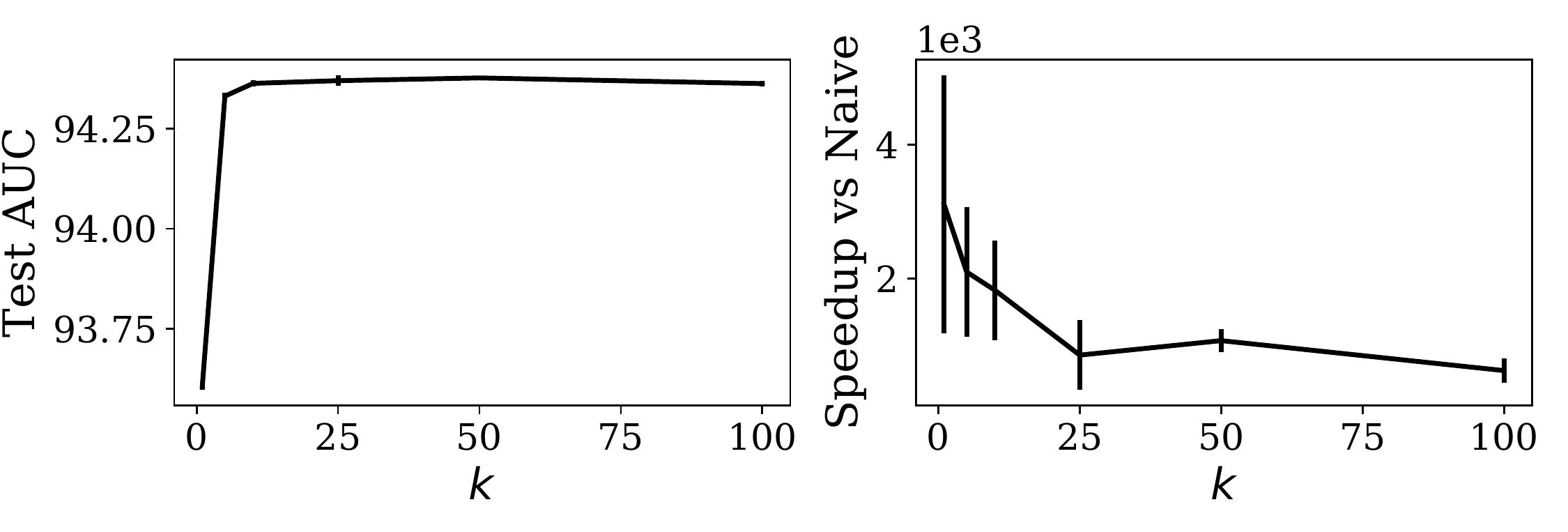}}\hfill
    \subfloat[Synthetic]{\includegraphics[width=\ww\textwidth]{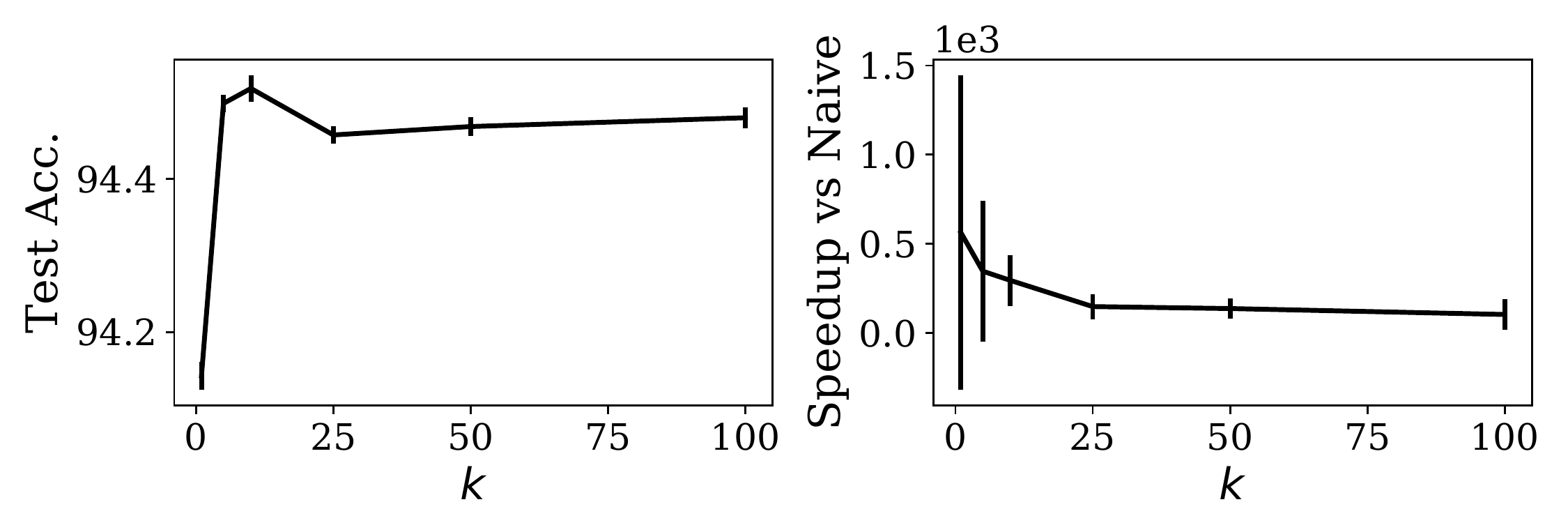}}\hfill
    \subfloat[Higgs: G-DaRE RF exceeded available memory for the settings in which $k \in \{25, 50, 100\}.$]{\includegraphics[width=\ww\textwidth]{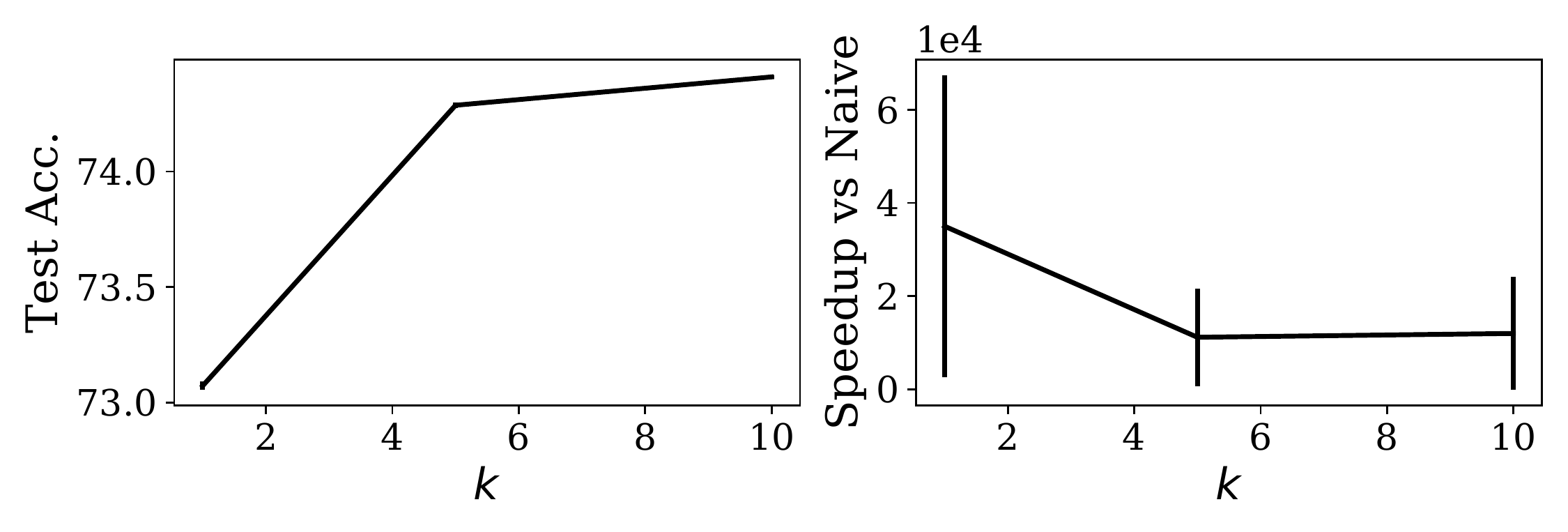}}\hfill
    \caption{Effect of $k$ on deletion efficiency.}
    \label{fig:appendix_k}
\end{figure}

\newpage~\newpage

\section{Additional Experiments}
\label{appendix_sec:additional_experiments}

\subsection{Entropy as the Split Criterion}
\label{appendix_subsec:entropy}

We repeat our experiments using entropy instead of Gini index as the split criterion and find similar results as when using Gini index. In terms of predictive performance, we find nearly identical results to those in Table~\ref{tab:predictive_performance} with selected hyperparameters shown in Table~\ref{tab:entropy_hyperparameters}. As for deletion efficiency, a summary of the deletion efficiency results is in Table~\ref{tab:entropy_deletion_efficiency_summary}; overall, we find the same trends as those in Table~\ref{tab:deletion_efficiency_summary}.

\begin{table}[h]
\centering
\caption{Hyperparameters selected using entropy as the split criterion.}
\vskip 0.15in
\label{tab:entropy_hyperparameters}
\begin{tabular}{@{}lrrrrrrr@{}}
\toprule
& \multicolumn{3}{c}{\textbf{G-DaRE} \& \textbf{R-DaRE}}
& \multicolumn{4}{c}{\textbf{R-DaRE Only}} \\
\cmidrule(lr){2-4}\cmidrule(lr){5-8}
\textbf{Dataset} &
  $T$ & $d_{\max}$ & $k$ &
  \begin{tabular}[c]{@{}r@{}}\dr\\ \textbf{(0.1\%)}\end{tabular}  &
  \begin{tabular}[c]{@{}r@{}}\dr\\ \textbf{(0.25\%)}\end{tabular} &
  \begin{tabular}[c]{@{}r@{}}\dr\\ \textbf{(0.5\%)}\end{tabular}  &
  \begin{tabular}[c]{@{}r@{}}\dr\\ \textbf{(1.0\%)}\end{tabular}  \\ \midrule
Surgical           & 100 & 20 & 50       & 1  & 1  & 2  & 4  \\
Vaccine            & 250 & 20 & 5        & 6  & 9  & 11 & 15 \\
Adult              & 50  & 20 & 5        & 9  & 12 & 14 & 15 \\
Bank Marketing     & 100 & 10 & 10       & 1  & 1  & 3  & 4  \\
Flight Delays      & 250 & 20 & 50       & 1  & 3  & 5  & 10 \\
Diabetes           & 100 & 20 & 5        & 4 & 10  & 11 & 14 \\
No Show            & 250 & 20 & 10       & 1  & 3  & 6  & 9  \\
Olympics           & 250 & 20 & 5        & 0  & 1  & 2  & 4  \\
Census             & 100 & 20 & 25       & 5  & 8  & 11 & 15 \\
Credit Card        & 250 & 10 & 25       & 1  & 2  & 3  & 4  \\
CTR                & 100 & 10 & 25       & 2  & 3  & 4  & 6  \\
Twitter            & 100 & 20 & 5        & 3  & 5  & 8  & 11 \\
Synthetic          & 50  & 20 & 10       & 1  & 2  & 3  & 6  \\
Higgs              & 50  & 20 & 10       & 0  & 2  & 5  & 8  \\
\bottomrule
\end{tabular}
\end{table}

\begin{table}[h]
\centering
\caption{Summary of the deletion efficiency results using entropy as the split criterion.}
\vskip 0.15in
\label{tab:entropy_deletion_efficiency_summary}
\begin{tabular}{lrrr}
\toprule
\textbf{Model} & \textbf{Min.} & \textbf{Max.} & \textbf{G.\ Mean} \\ \midrule
\multicolumn{4}{l}{\textbf{Random Adversary}} \\
G-DaRE               &  81x &  9,986x &   715x \\
R-DaRE (tol=0.1\%)   &  93x &  9,986x &   834x \\
R-DaRE (tol=0.25\%)  &  93x & 21,104x & 1,177x \\
R-DaRE (tol=0.5\%)   &  99x & 16,897x & 1,265x \\
R-DaRE (tol=1.0\%)   & 128x & 37,953x & 1,819x \\
\midrule
\multicolumn{4}{l}{\textbf{Worst-of-1000 Adversary}} \\
G-DaRE               & 19x &   790x &  76x \\
R-DaRE (tol=0.1\%)   & 24x &   790x & 101x \\
R-DaRE (tol=0.25\%)  & 25x & 1,348x & 135x \\
R-DaRE (tol=0.5\%)   & 29x & 1,473x & 175x \\
R-DaRE (tol=1.0\%)   & 37x & 1,783x & 262x \\ \bottomrule
\end{tabular}
\end{table}

\end{document}